\documentclass[11pt,english]{article}
\usepackage[utf8]{inputenc}
\usepackage{amsmath,amsthm,amssymb,amsfonts}
\usepackage{comment}
\usepackage{algorithm}
\usepackage[noend]{algcompatible}
\algnewcommand\algorithmicreturn{\textbf{Return}}
\algnewcommand\RETURN{\State \algorithmicreturn}%

\usepackage{graphicx}
\usepackage[font=footnotesize]{caption}
\usepackage{xcolor}
\usepackage{nicefrac}
\usepackage[shortlabels]{enumitem}
\DeclareMathOperator*{\argmin}{arg\,min}

\usepackage{a4wide}
\usepackage{chngpage}
\usepackage{authblk}

\newtheorem{assumption}{Assumption}
\newtheorem{remark}{Remark}
\newtheorem{theorem}{Theorem}

\newtheorem{lemma}{Lemma}
\newtheorem{definition}{Definition}

\usepackage[T1]{fontenc}
\usepackage{babel}
\usepackage{threeparttable}
\usepackage{booktabs}
\usepackage{etoolbox}
\appto\TPTnoteSettings{\footnotesize}

\usepackage[most]{tcolorbox}

\tcbset{colback=blue!10!white, colframe=blue!50!black, 
        highlight math style= {enhanced, 
            colframe=blue,colback=blue!10!white,boxsep=0pt}
}

\usepackage{mathtools,nccmath,textcomp}
\usepackage[round]{natbib}

\usepackage[backref=page]{hyperref}
\renewcommand*{\backref}[1]{}
\renewcommand*{\backrefalt}[4]{%
    \ifcase #1 %
    \or        (Cited on page~#2.)%
    \else      (Cited on pages~#2.)%
    \fi}
\usepackage{bbding}
\usepackage{pifont}

\usepackage{tikz}
\usepackage{pgfplots}
\pgfplotsset{compat=1.18}
\usepackage{subcaption}

\newcommand{\bx}{\mathbf{x}}
\newcommand{\bz}{\mathbf{z}}
\newcommand{\by}{\mathbf{y}}

\newcommand{\R}{\mathbb{R}}
\newcommand{\D}{\mathcal{D}}
\newcommand{\C}{\mathcal{C}}
\newcommand{\N}{\mathbb{N}}
\newcommand{\calN}{\mathcal{N}}
\newcommand{\bL}{\mathbf{L}}

\title{A Unified Framework for Center-based Clustering of Distributed Data}

\author[1]{Aleksandar Armacki
}
\author[2]{Dragana Bajovi\'{c}}
\author[3]{Du\v{s}an Jakoveti\'{c}}
\author[1]{Soummya Kar}
\affil[1]{Carnegie Mellon University, Pittsburgh, PA, USA\\ \texttt{\{aarmacki,soummyak\}@andrew.cmu.edu }}
\affil[2]{Faculty of Technical Sciences, University of Novi Sad, Novi Sad, Serbia\\ \texttt{dbajovic@uns.ac.rs}}
\affil[3]{Faculty of Sciences, University of Novi Sad, Novi Sad, Serbia\\ \texttt{dusan.jakovetic@dmi.uns.ac.rs}}

\date{}

\begin{document}

\maketitle

\begin{abstract}
     We develop a family of distributed center-based clustering algorithms that work over networks of users. In the proposed scenario, users contain a local dataset and communicate only with their immediate neighbours, with the aim of finding a clustering of the full, joint data. The proposed family, termed Distributed Gradient Clustering (DGC-$\mathcal{F}_\rho$), is parametrized by $\rho \geq 1$, controling the proximity of users' center estimates, with $\mathcal{F}$ determining the clustering loss. Our framework allows for a broad class of smooth convex loss functions, including popular clustering losses like $K$-means and Huber loss. Specialized to popular clustering losses like $K$-means and Huber loss, DGC-$\mathcal{F}_\rho$ gives rise to novel distributed clustering algorithms DGC-KM$_\rho$ and DGC-HL$_\rho$, while novel clustering losses based on Logistic and Fair functions lead to DGC-LL$_\rho$ and DGC-FL$_\rho$. We provide a unified analysis and establish several strong results, under mild assumptions. First, we show that the sequence of centers generated by the methods converges to a well-defined notion of fixed point, under any center initialization and value of $\rho$. Second, we prove that, as $\rho$ increases, the family of fixed points produced by DGC-$\mathcal{F}_\rho$ converges to a notion of consensus fixed points. We show that consensus fixed points of DGC-$\mathcal{F}_{\rho}$ are equivalent to fixed points of gradient clustering over the full data, guaranteeing a clustering of the full data is produced. For the special case of Bregman losses, we show that our fixed points converge to the set of Lloyd points. Extensive numerical experiments on synthetic and real data confirm our theoretical findings, show strong performance of our methods and demonstrate the usefulness and wide range of potential applications of our general framework, such as outlier detection.
\end{abstract}

\section{Introduction}

Clustering is an unsupervized learning problem, where the goal is to group the data based on a similarity criteria, without having any prior knowledge of the underlying distribution or the true number of groups, e.g., \cite{clustering-survey,JAIN2010651}. Applications of clustering have a wide range, including domains such as marketing research, text classification, anomaly detection, and biomedical applications, e.g., \cite{clustering-marketing,dhillon2003adivisive,banerjee_anomaly,clustering-biomedical,Huber_clust}. There are many different approaches to clustering, such as center-based \cite{kmeans_awasthi}, density \cite{density-clust} and spectral clustering \cite{kumar_kmeans,kmeans_awasthi}, to name a few. From the point of assignment, clustering  can be hard or soft, with hard clustering assigning each sample to only one cluster, while soft clustering outputs the probability of a sample belonging to each cluster. In this paper we are interested in the center-based hard clustering problem. Formally, for a given dataset $\D = \{y_1,\ldots,y_N \} \subset \R^d$, the problem of partitioning $\D$ into $K$ disjoint clusters is given by
\begin{equation}\label{eq:general-clust}
    \min_{\bx \in \R^{Kd},C \in \mathcal{C}_{K,\D}}\hspace{-0.8em}H(\bx,C) = \sum_{k \in [K]}\sum_{r \in C(k)}w_{r}f(x(k), y_r),
\end{equation} where $\bx = \begin{bmatrix}x(1)^\top & \ldots & x(K)^\top  \end{bmatrix}^\top$ is the vector stacking the $K$ centers $x(k) \in \R^d$, $\C_{K,\D}$ is the set of all $K$-partitions of $\D$, i.e., $C \in \C_{K,\D}$ is a $K$-tuple $C = \begin{pmatrix}C(1), \ldots, C(K) \end{pmatrix}$, with $C(k) \subseteq \D$\footnote{In a slight abuse of notation, we will also use $\D$ to denote the set of indices of the data, i.e., $\D = [N]$.}, $C(k) \cap C(l) = \emptyset$ and $\cup_{k \in [K]}C(k) = \D$, $w_r \in (0,1)$ is a fixed weight associated with sample $r$, such that $\sum_{r \in [N]}w_r = 1$, and $f: \R^d \times \R^d \mapsto [0,\infty)$ a loss function. For example, setting $f$ to be the squared Euclidean norm, one recovers the renowned $K$-means clustering problem \cite{Lloyd,awasthi2016center}. In general, the problem \eqref{eq:general-clust} is highly non-convex and NP-hard, see \cite{kmeans-stability,Kmedians-NP,Vattani2010TheHO,awasthi2015hardness}. As such, the best one can hope for is finding a stationary point of \eqref{eq:general-clust}, with various schemes guaranteeing convergence to stationary points, e.g., \cite{Macqueen67somemethods,Lloyd,JMLR:v6:banerjee05b,Huber_clust,pmlr-v162-armacki22a}. 

The amount of available data has seen a rapid increase recently. For example, it is estimated that $147$ zetabytes of data will be produced worldwide in 2024, a growth of over a $100\%$ compared to the amount in 2020 \cite{data-generated}. Such trends often render storing and processing the data at a single location impossible, making \emph{distributed algorithms instrumental}. Distributed learning is a paradigm where multiple users collaborate to train a joint model. It has been attracting great interest recently, as it offers benefits like privacy, with users' data stored locally, only exchanging model parameters \cite{pmlr-v54-mcmahan17a}. Another benefit is decreased computation and storage burden, as data is handled locally, with smaller chunks of data to be parsed, e.g., \cite{yang-trading,jakovetic-pd}. Communication-wise, distributed algorithms are client-server (CS) \cite{pmlr-v54-mcmahan17a,sahu-fl,kairouz-fl}, or peer-to-peer (P2P) \cite{consensus+innovation1,sayed-networks,vlaski_et_al}, with some important differences. In particular, the CS setup assumes that each user communicates directly with a central server, which coordinates the learning process and aggregates the local updates sent by the users. As all the users are connected to the server, this approach in effect behaves like a centralized learning algorithm, while providing the benefits of privacy, by maintaining the data locally at each user. On the other hand, the P2P setup assumes that there is no central user (i.e., server) that can communicate with all the other users and coordinate the learning process. Instead, users communicate directly with other users from their immediate neighbourhood, defined by an underlying communication graph. As such, P2P methods rely on a blend of local model updates and consensus-style communication, to ensure that information is diffused across the network, e.g., \cite{consensus+innovation1,consensus+innovation2,sayed-networks}. Since no single point of failure (i.e., server) exists in P2P methods, they are inherently more robust to user failure, dropout, as well as unreliable and adversarial users, \cite{tsianos-practical,dist-adverserial,yu-secure}. Figure \ref{fig:network-topologies} visualizes the client-server and peer-to-peer communication setups. We study the peer-to-peer setup and will use the term ``distributed'' to refer to it. We study the P2P setup and will use the term ``distributed'' to refer to this setup.

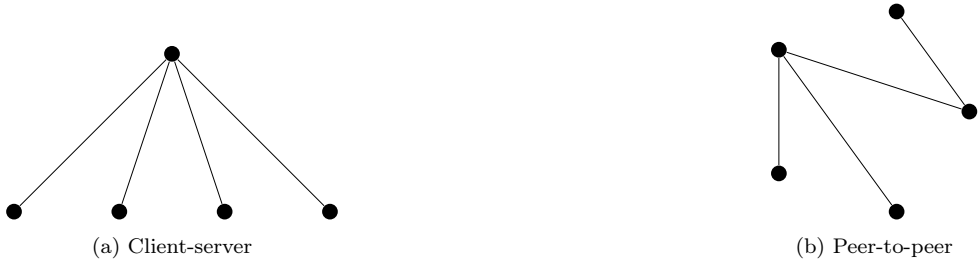
\begin{figure}
  \centering
  \begin{subfigure}{0.4\textwidth}
    \centering
    \scalebox{0.7}{
    \begin{tikzpicture}
      \node[circle, fill, inner sep=3pt] (center) at (5,3) {};
      \foreach \i in {1,...,4} {
        \node[circle, fill, inner sep=3pt] (node\i) at (2*\i,0) {};
        \draw (center) -- (node\i);
      }
    \end{tikzpicture}
    }
    \caption{Client-server}
    \label{fig:star-topology}
  \end{subfigure}
  \hfill
  \begin{subfigure}{0.4\textwidth}
    \centering
    \scalebox{0.7}{
    \begin{tikzpicture}
      \foreach \i in {1,...,5} {
        \node[circle, fill, inner sep=3pt] (node\i) at ({72 * (\i - 1)}:2) {};
      }
      \draw (node1) -- (node2);
      \draw (node1) -- (node3);
      \draw (node3) -- (node5);
      \draw (node4) -- (node3);
    \end{tikzpicture}
    }
    \caption{Peer-to-peer}
    \label{fig:connected-graph}
  \end{subfigure}
  \caption{An example of client-server and peer-to-peer setups in distributed learning. Vertices represents users and edges represent bidirectional communication links.}
  \label{fig:network-topologies}
\end{figure}

The distributed setting presents a unique challenge for clustering, as users store their data locally, making it difficult to produce a clustering of the full, joint dataset. Many recent works focus on the CS setting, e.g., \cite{fedclust3,fedclust4,fedclust5,one_shot_clust,fedclust1,fedclust2}. In comparison, P2P clustering has been studied in proportionally much smaller body of work. \emph{In this paper, we develop a unified approach for P2P (hard) clustering}. We do so by first proposing a general clustering formulation specifically designed for distributed setups, giving rise to distributed versions of popular (centralized) clustering formulations, e.g., $K$-means \cite{Lloyd} and Huber clustering \cite{Huber_clust}. Next, we develop a method that solves the general problem, provably converges, results in novel clustering algorithms when applied to specific losses and produces a clustering of the full data in P2P networks. Our method is general, easy to implement and exhibits strong theoretical and practical performance (see Sections \ref{sec:methods}-\ref{sec:num}).

\textbf{Literature review}. P2P clustering methods have been proposed in \cite{kmean-dynamic,dist-clust-wsn,dist-coresets,oliva2013distributed,dist-fuzzy,kar2019clustering,dist-soft}. The paper \cite{kmean-dynamic} proposes approximate $K$-means algorithms for both P2P and CS networks, with a theoretical study of asymptotic performance guarantees provided for the CS version of the algorithm. The performance guarantees are measured in terms of the deviation of the resulting clustering from a clustering produced by running the $K$-means algorithm on the full, centralized data. Works \cite{dist-clust-wsn,dist-fuzzy,dist-soft} study soft and hard $K$-means. Only the method in \cite{dist-clust-wsn}, built on the Alternating Directions Method of Multipliers (ADMM) framework, shows provable convergence, establishing asymptotic convergence of the sequence of centers to a Karush-Kuhn-Tucker (KKT) point, guaranteeing asymptotic consensus and convergence to a local minima of the hard $K$-means clustering problem. In \cite{dist-coresets}, the authors study $K$-means and $K$-medians problems and propose methods with provable guarantees. In particular, the authors provide guarantees for coreset construction over P2P networks for both $K$-means and $K$-median problems. Using these coresets, the authors then show it is possible to design methods with provable constant approximation guarantees for distributed $K$-means and $K$-medians. Work \cite{oliva2013distributed} studies $K$-means in the special case where users have a single sample. The authors in \cite{kar2019clustering} propose a parametric family of $K$-means methods, establishing asymptotic convergence of centers to Lloyd points, i.e., local minima of the centralized $K$-means problem.Finally, it is worth mentioning \cite{pmlr-v162-armacki22a}, who propose a general framework for gradient-based\footnote{The term ``gradient-based'' is motivated by the center update rule, as is typical in naming center-based methods, like $K$-means \cite{Lloyd}, $K$-medians \cite{arora_kmedians} or Huber clustering \cite{Huber_clust}. Centralized gradient clustering \cite{pmlr-v162-armacki22a} and our distributed gradient clustering are not pure gradient descent methods, as the center-based clustering problem is not jointly differentiable, see \eqref{eq:general-clust}-\eqref{eq:distributed-gen} ahead.} clustering in centralized settings. Our work can be seen as its distributed counterpart, with important differences discussed in Section \ref{sec:methods}. Note that the literature on distributed hard clustering is lacking in methods beyond $K$-means, as the efforts almost exclusively focus on designing variants of $K$-means. It is well known in the centralized setting that, in applications where different properties are desired, e.g., robustness to outliers, \emph{methods beyond $K$-means are required} \cite{Huber_clust,arora_kmedians,JMLR:v6:banerjee05b} (see also a discussion on the usefulness and intuition behind using method beyond $K$-means in Appendix \ref{app:intuition}). While \cite{dist-coresets} provides a distributed $K$-median method, their approach is built on the idea of designing a coreset \cite{coresets}, i.e., creating a set that approximates the full data and using it for training. The construction of the coreset is costly and requires running involved approximation algorithms and users communicating subsets of data. As such, the methods in \cite{dist-coresets} can incur high communication and storage costs, as the size of the coreset scales with the number of users $m$ and desired approximation quality\footnote{The size of coreset is $\mathcal{O}\big(\frac{1}{\epsilon^4}(Kd + \log\frac{1}{\gamma}) + mK\log\frac{mK}{\gamma}\big)$, where $\gamma \in (0,1)$ and $\epsilon > 0$ are the success probability and approximation quality.}. Moreover, their approach requires users to share the local data, which can make them reluctant to participate, due to privacy concerns. Works \cite{dist-clust-wsn,kar2019clustering,pmlr-v162-armacki22a} are closest to ours, as they provide iterative methods which only exchange local center estimates, with asymptotic convergence guarantees. Compared to them, ours is the only work that simultaneously works in P2P networks and supports costs beyond $K$-means.\footnote{The work \cite{dist-clust-wsn} provides methods for soft clustering, however, only a $K$-means metod is designed for distributed hard clustering.} We provide detailed comparisons with these methods in Section \ref{sec:methods} and in numerical simulations in Section \ref{sec:num}.

Another line of work related to ours is that of first-order methods for distributed optimization, e.g., \cite{nedic-subgrad,nedic-geometric,shi-extra,scutari-next,kun_dgd,jakovetic-fast,jakovetic-rates,xin-general,xin-fast,xin-incr,swenson-dsgd}. As the framework proposed in this paper is gradient-based and used to solve a distributed optimization problem (see Sections \ref{sec:problem}-\ref{sec:methods}), it is related to first-order distributed optimization methods, with some key differences. First, we consider the specific problem of clustering, which, apart from \cite{kar2019clustering}\footnote{The analysis in \cite{kar2019clustering} is tailored to the method proposed therein, oblivious to the fact that it is a first-order method. As we discuss in Section \ref{sec:methods}, their method is a special case of ours, applied to the $K$-means cost.}, has not been studied in the context of first-order methods. Second, due to optimizing over both centers and clusters (see Sections \ref{sec:problem}-\ref{sec:methods}), the problem considered in this paper is a combined continuous (centers) and discrete (clusters) problem and the analysis techniques typically used in distributed optimization are not applicable, requiring novel approaches for convergence guarantees.

\textbf{Contributions}. Our contributions are as follows.
\begin{itemize}[leftmargin=*]
\item We propose a general approach for clustering data over P2P networks, dubbed DGC-$\mathcal{F}_\rho$. Our approach is applicable for smooth, convex loss functions, e.g., $K$-means, Huber, Logistic and Fair functions, and general distance metrics, e.g., Euclidean and Mahalanobis distance. DGC-$\mathcal{F}_\rho$ works over any connected communication graph and users only exchange center estimates, with data remaining private. 
    
\item We establish convergence guarantees in the following regimes. For fixed $\rho$, we show that DGC-$\mathcal{F}_\rho$ converges to aptly defined fixed points under any center initialization, making it amenable to initialization schemes like $K$-means$++$, while the clusters \emph{converge in finite time}. As $\rho$ grows, we show that fixed points of DGC-$\mathcal{F}_\rho$ converge to the set of aptly defined consensus fixed points.

\item We show that as $\rho\rightarrow\infty$, the cluster center estimates attain consensus, thus \emph{guaranteeing that clusters converge to a clustering of the full data, for sufficiently large $\rho$}. In the case of Bregman losses (e.g., $K$-means), we show that these limiting consensus fixed points reduce to the classical notion of Lloyd fixed points associated with hard clustering. No assumptions are made on users' data similarity, with data across users possibly highly heterogeneous.
    
\item We extensively verify the performance of our methods on both synthetic and real datasets, showing strong performance across a myriad of scenarios. Moreover, we demonstrate that our method DGC-HL with Huber loss can be used for outlier detection, further underlining the importance of considering methods beyond $K$-means, as well as the usefulness and wide range of potential applications of our general framework for distributed center-based clustering.
\end{itemize}

\textbf{Paper organization}. The rest of the paper is organized as follows. Section \ref{sec:problem} introduces the problem consider in the paper. Section \ref{sec:methods} outlines the proposed family of methods. Section \ref{sec:main} presents the main results. Section \ref{sec:num} provides numerical results. Section \ref{sec:conclusion} concludes the paper. Appendix contains additional details and proofs. We introduce the notation in the remainder of this section.

\textbf{Notation}. The spaces of real numbers and $d$-dimensional vectors are denoted by $\R$ and $\R^d$. The set of non-negative integers is denoted by $\N$. The set of integers up to and including $M$ is denoted by $[M] = \{1,\ldots,M\}$. For a set $A$, $\overline{A}$ denotes the closure of $A$, while $|A|$ denotes the number of elements of $A$. We use $\langle \cdot, \cdot \rangle$ and $\| \cdot \|$ to denote the Euclidean inner product and the induced vector/matrix norm. $\nabla_x f(x,y)$ denotes the gradient of $f$ with respect to $x$. We use $1_d$ and $I_d$ to denote the $d$-dimensional vector of ones and $d \times d$ identity matrix. $\otimes$ denotes the Kronecker product, $A^\top$ denotes transposition and $\overline{\lambda}(A)$ denotes the largest eigenvalue of $A$. $\mathcal{O}(\cdot)$ is the ``big O'', i.e., $a_n = \mathcal{O}(b_n)$ implies there exist $C > 0$ and $n_0 \in \N$, such that $a_n \leq Cb_n$, for all $n \geq n_0$, for $a_n, b_n \geq 0$. Superscripts denote the iteration counter, subscripts denote the user, while the value in the brackets corresponds to the particular center/cluster, e.g., $x^t_i(k)$ stands for the $k$-th center of user $i$, at iteration $t$.

\section{Problem formulation}\label{sec:problem}

Consider a network of $m > 1$ users connected over a graph $G = (V,E)$, where $V = [m]$ is the set of vertices (i.e., users), $E$ is the set of (undirected) edges connecting them, such that $\{i,j\} \in E$ if and only if users $i$, $j$ can communicate. Let $\D_i = \{y_{i,1},\ldots,y_{i,N_i}\}$, $w_{i,r} \in (0,1)$  be the local data and weight associated with the $r$-th point of user $i$, with $\sum_{i \in [m]}N_i = N$ and $\sum_{i,r }w_{i,r} = 1$. In this setup, \eqref{eq:general-clust} is equivalent to
\begin{equation}\label{eq:general-constr}
    \min_{\substack{\bx_i \in \R^{Kd}, \: C_i \in \C_{K,\D_i} \\ \text{subject to }\bx_1 = \ldots = \bx_m}}\sum_{i \in [m]}\sum_{k \in [K]}\sum_{r \in C_i(k)}\hspace{-0.75em}w_{i,r}f(x_i(k),y_{i,r}),
\end{equation} where $\mathcal{C}_{K,\D_i}$ is the set of $K$-partitions of $\D_i$. Formulation \eqref{eq:general-constr} ensures that a clustering of the joint dataset is produced, by synchronizing the center estimates across users via the constraint $\bx_1 = \ldots = \bx_m$, i.e., enforcing that the centers are the same across all users. For \eqref{eq:general-constr} to be well defined and solvable in distributed fashion, we assume the following.

\begin{assumption}\label{asmpt:data}
    The full dataset $\D = \cup_{i \in [m]}\D_i$ has at least $K$ distinct samples.
\end{assumption}

\begin{assumption}\label{asmpt:graph}
    The graph $G = (V,E)$ is connected.
\end{assumption}

\begin{remark}
    Assumption \ref{asmpt:data} is natural, as we aim to find $K$ clusters. No assumptions are made on the local datasets $\D_i$, which can be highly heterogeneous across users, e.g., having different sizes of datasets or containing different classes. Assumption \ref{asmpt:graph} ensures that \eqref{eq:general-constr} can be solved in a distributed fashion, by P2P communication only. It is standard in distributed literature, see \cite{vlaski_et_al} and references therein.
\end{remark}

Note that \eqref{eq:general-constr} is a constrained problem, requiring either global synchronization of users' center estimates, or an involved primal-dual scheme, e.g., as in \cite{dist-clust-wsn}, to be solved. To make it amenable to a simple first-order approach and local communication only, we consider the relaxation
\begin{align}
    \min_{\substack{\bx \in \R^{Kmd}, \\ C \in \C_{m,K,\D}}} J_\rho(\bx,C) &= \sum_{i \in [m]}\sum_{k \in [K]}\Big[ \frac{1}{\rho}\sum_{r \in C_i(k)}w_{i,r}f(x_i(k),y_{i,r}) + \frac{1}{2}\sum_{j \in \calN_i}\|x_i(k) - x_j(k)\|^2 \Big], \label{eq:general-decentr}
\end{align} where $\mathcal{C}_{m,K,\D}$ is the set of all clusterings of the data, i.e., for $C \in \mathcal{C}_{m,K,\D}$, we have $C = (C_1,\ldots,C_m)$, with $C_i \in \mathcal{C}_{K,\D_i}$, $\calN_i = \left\{j \in V: \{i,j \} \in E\right\}$ is the set of neighbours of user $i$ (not including $i$), while $\rho \geq 1$ is a tunable parameter. The formulation \eqref{eq:general-decentr} relaxes \eqref{eq:general-constr}, by considering an unconstrained problem that penalizes the difference of centers among neighbouring users and controls the trade-off between center estimation and proximity, via $\rho$. A similar relaxation was considered in \cite{kar2019clustering}, for the case $f(x,y) = \|x - y\|^2$.

\begin{remark}
    The primary motivation for considering relaxation \eqref{eq:general-decentr} is the ability to solve the relaxed problem in a distributed manner, while enforcing consensus among users via the graph Laplacian term. The idea behind multiplying the clustering part in \eqref{eq:general-decentr} by a factor $1/\rho$, is to slowly strengthen the effect of consensus as $\rho \rightarrow \infty$, ensuring consensus is achieved. An further benefit of \eqref{eq:general-decentr} is the fact that it is much easier to find stationary points of unconstrained problems. As such, it is often faster to obtain a sequence of stationary points of a succession of relaxations, rather than solving the original constrained problem, which is the idea behind a class of methods known as \emph{penalty methods}, see, e.g., \cite{bertsekas2014constrained}.
\end{remark}

\begin{remark}
    The choice of penalty factor $1/2$ for the consensus part in \eqref{eq:general-decentr} is arbitrary and can be replaced by any other constant or adaptive factor, such as $(1 - 1/\rho)$, with the asymptotic behaviour of stationary points remaining the same. As we are interested in the asymptotic behaviour of stationary points when studying consensus guarantees, this implies that we can equivalently consider such a relaxation of problem \eqref{eq:general-constr}. 
\end{remark}

The formulation \eqref{eq:general-decentr} is very general and includes a myriad of clustering loss functions, with some examples given next.\\

\noindent \textbf{Example 1}. \emph{Distributed Bregman clustering}: if the loss in \eqref{eq:general-decentr} is a Bregman distance \cite{BREGMAN1967200}, $f(x,y) = \psi(y) - \psi(x) - \langle \nabla \psi(x),y-x\rangle$, with $\psi: \R^d \mapsto \R$ strictly convex and differentiable, we get a novel problem formulation of distributed Bregman clustering. For the special case $f(x,y) = \|x - y\|^2$, we recover the distributed $K$-means formulation from \cite{kar2019clustering}.\\

\noindent \textbf{Example 2}. \emph{Distributed Huber clustering}: if the loss in \eqref{eq:general-decentr} is the Huber loss \cite{huber_loss}, $f(x,y) = \phi_\delta(\|x - y\|)$, where $\phi_\delta: \R \mapsto [0,\infty)$, for some $\delta > 0$, is given by 
\begin{equation}\label{eq:Huber}
    \phi_\delta(x) = \begin{cases}
        \frac{x^2}{2}, & |x| \leq \delta \\
        \delta |x| - \frac{\delta^2}{2}, & |x| > \delta
    \end{cases},
\end{equation} we get a novel formulation of distributed Huber clustering.\\

\noindent \textbf{Example 3}. \emph{Distributed Logistic clustering}: if the loss in \eqref{eq:general-decentr} is the Logistic loss, i.e., $f(x,y) = \log(1 + \exp(\|x - y\|^2))$, we get a novel problem of distributed Logistic clustering.\\

\noindent \textbf{Example 4}. \emph{Distributed Fair clustering}: if the loss in \eqref{eq:general-decentr} is the ``Fair'' loss $f(x,y) = h_\eta(\|x - y\|)$, where $h_\eta(x) = 2\eta^2(\nicefrac{x^2}{\eta} - \log(1 + \nicefrac{x^2}{\eta})$, for $\eta > 0$, e.g., \cite{Fair-loss}, we get a novel problem of distributed Fair clustering.\\

Note that all formulations are novel, from the perspective of distributed clustering. While Bregman and Huber clustering are popular clustering formulations in the centralized setting, e.g., \cite{JMLR:v6:banerjee05b,Huber_clust,pmlr-v162-armacki22a}, \emph{to the best of our knowledge, losses like Logistic and Fair have not been considered previously, even in the centralized setting}. Using the graph Laplacian matrix $L = D - A$, where $D, A \in \R^{m \times m}$ are the degree and adjacency matrices, see \cite{chung1997spectral,cvetkovic_rowlinson_simic_1997}, and letting $\bx = \begin{bmatrix}\bx_1^\top & \ldots & \bx_m^\top\end{bmatrix}^\top \in \R^{Kmd}$ be the vector stacking users' center estimates $\bx_i \in \R^{Kd}$, we can represent \eqref{eq:general-decentr} as
\begin{equation}\label{eq:distributed-gen}
        \min_{\bx \in \R^{Kmd},  C \in \C_{m,K,\D}}\hspace{-1em}J_\rho(\bx,C) = \frac{1}{\rho}J(\bx,C) + \frac{1}{2}\langle \bx,\bL\bx\rangle,
\end{equation} where $J(\bx,C) = \sum_{i \in [m]}H(\bx_i,C_i)$, with $\bL = L \otimes  I_{Kd}$. To solve \eqref{eq:distributed-gen}, we make the following assumptions.

\begin{assumption}\label{asmpt:coerc} 
The loss $f$ is coercive, convex and $\beta$-smooth with respect to the first argument, i.e., for all $x,y,z\in \R^d$, we have $\lim_{\|x\| \rightarrow \infty}f(x,y) = \infty$, and 
\begin{equation*}
     0 \leq f(x,y) - f(z,y) - \langle \nabla f(z,y), x - z\rangle \leq \frac{\beta}{2}\|x - z\|^2.
\end{equation*}
\end{assumption}    

\begin{remark}
    Assumption \ref{asmpt:coerc} ensures the loss function is well-behaved, with coercivity ensuring center estimates stay close to the dataset $\D$ by not allowing them to grow arbitrarily large, while convexity and smoothness are standard assumptions for gradient-based methods, e.g., \cite{lectures_on_cvxopt}. In the context of clustering, apart from our work and \cite{pmlr-v162-armacki22a}, $\beta$-smoothness has been used in \cite{paul2021uniform,ieee_ghosh}, however, these works differ in that they study statistical guarantees of centralized \cite{paul2021uniform} and client-server (CS) \cite{ieee_ghosh} clustering algorithms, whereas the current work studies convergence guarantees of a clustering algorithm over a finite dataset, distributed across a P2P network.
\end{remark}

\begin{remark}
    Note that Assumption 3 is a property of the \emph{loss function $f$, rather than the dataset we wish to cluster}. In other words, whether Assumption 3 is satisfied depends only on the choice of clustering loss $f$ and is completely independent of the specific dataset used for clustering. All the loss functions used in Examples 1-4 and throughout our work, namely $K$-means, Huber, logistic and fair loss, satisfy Assumption 3. For a formal proof of this claim, see Lemma \ref{lm:asmpt3} in Appendix \ref{app:asmpt3}.
\end{remark}

\section{Proposed family of methods}\label{sec:methods}

In this section we describe the DGC-$\mathcal{F}_\rho$ family of methods proposed to solve the general distributed clustering problem \eqref{eq:distributed-gen}. During training users maintain their current center and cluster estimates. The algorithm starts with users choosing their initial center estimates $\bx_i^0 \in \R^{Kd}$, $i \in [m]$. At iteration $t \geq 0$, each user $i\in [m]$ first forms the clusters, by finding a $k \in [K]$ for each data point $r \in \D_i$, such that the $k$-th center is the closest to the point $r$. For this purpose, we introduce a \emph{novel distance function} $g: \R^d \times \R^d \mapsto [0,\infty)$ (not necessarily the Euclidean distance), seeking $k \in [K]$ such that 
\begin{equation}\label{eq:reassign}
    g(x_i^t(k),y_{i,r}) \leq g(x_i^t(l),y_{i,r}), \: \text{for all } l \ne k,
\end{equation} and assigns $y_{i,r}$ to $C_i^{t+1}(k)$. Here, $g: \R^d \times \R^d \mapsto [0,\infty)$ is a distance function, related to the loss $f$ (see Assumptions \ref{asmpt:met}, \ref{asmpt:g&f} ahead). Next, the centers are updated, by performing $B \geq 1$ updates, i.e., for $b = 0,\ldots,B-1$
\begin{align}
    x_i^{t,b+1}&(k) = x_i^{t,b}(k) - \alpha\bigg( \underbrace{\sum_{j \in \calN_i} \left[x^{t,b}_i(k) - x^{t,b}_j(k) \right]}_{\text{consensus}} + \underbrace{\frac{1}{\rho}\sum_{r \in C_i^{t+1}(k)}w_{i,r}\nabla_x f\left(x_i^{t,b}(k),y_{i,r}\right)}_{\text{innovation}} \bigg), \label{eq:grad_local}
\end{align} where $x_i^{t,0}(k) = x_i^{t}(k)$ and $\alpha > 0$ is a sufficiently small fixed step-size. Finally, the new center is $x_i^{t+1}(k) = x_i^{t,B}(k)$, and the steps are repeated. The procedure is summarized in Algorithm \ref{alg:dist-grad-cl}. Center initialization performed at the outset of training can be done in an arbitrary manner, with each user allowed to initialize their own centers, requiring no synchronization. This allows for significant flexibility and implementing initialization algorithms, like distributed $K$-means$++$ \cite{dist-soft}. Steps 2-5 in Algorithm \ref{alg:dist-grad-cl} outline the cluster update steps, while Steps 7-10 outline the center update steps, using the consensus + innovation framework \cite{consensus+innovation1,consensus+innovation2}.

\begin{remark}\label{rmk:gradient-intuition}
    The main idea behind gradient-based clustering is that it allows us to unify several clustering methods using a general and simple update rule, namely (local) gradient descent. While methods like $K$-means \cite{Lloyd} or Huber clustering \cite{Huber_clust} are usually treated separately and require designing specialized algorithms and update rules, using gradient clustering allows us to unify these seemingly unrelated methods in a simple and elegant framework. Moreover, by including two novel losses, \emph{fair} and \emph{logistic} loss, which, to the best of our knowledge, \emph{have not been considered in the context of clustering}, we show that gradient-based clustering leads to rise of novel clustering methods. Another significant benefit of this approach stems from the fact that while it is not always possible to design closed-form updates for clustering methods, it is always possible to perform updates via gradient descent. For a detailed discussion along these lines, as well as importance and intuition behind methods beyond $K$-means, see Appendix \ref{app:intuition}.
\end{remark}

Note that DGC-$\mathcal{F}_\rho$ uses different functions for cluster assignment and center updates. We assume the following on the relationship between $g$ and $f$.

\begin{algorithm}[!tb]
\caption{DGC-$\mathcal{F}_\rho$}
\label{alg:dist-grad-cl}
\begin{algorithmic}[1]
   \REQUIRE{Step-size $\alpha > 0$, penalty parameter $\rho \geq 1$, number of rounds $T \geq 1$, number of center updates $B\geq 1$, initial centers $\bx_i^0 = \begin{bmatrix}x_i^0(1)^\top & \ldots & x_i^0(K)^\top \end{bmatrix}^\top \in \R^{Kd}$, $i \in [m]$.}
   \FOR{all users $i \in [m]$ in parallel, in round t = 0,1,\ldots,T-1}:
        \STATE Set $C_i^{t+1}(k) \leftarrow \emptyset$, for all $k \in [K]$; 
        \FOR{each $r \in [N_i]$}:
            \STATE Find $k \in [K]$ such that, for all $l \in [K]$: $g(x_i^t(k),y_{i,r}) \leq g(x_i^t(l),y_{i,r})$;
            \STATE Update the cluster: $C_i^{t+1}(k) \leftarrow C_i^{t+1}(k) \cup \{r\}$;
        \ENDFOR
        \STATE Set $x_i^{t,0}(k) \leftarrow x_i^t(k)$;
        \FOR{all clusters $k \in [K]$ in parallel and center update rounds $b = 0,\ldots,B-1$}:
            \STATE Exchange the current center estimates $x_i^{t,b}(k)$ and $x_j^{t,b}(k)$ with neighbours $j \in \calN_i$;
            \STATE $x_i^{t,b+1}(k) \leftarrow x^{t,b}_i(k) - \alpha\left(\frac{1}{\rho}\sum_{r \in C_i^{t+1}(k)}\nabla_x f(x_i^{t,b}(k),y_{i,r}) + \sum_{j \in \calN_i}[x_i^{t,b}(k) - x_j^{t,b}(k)] \right)$;
        \ENDFOR
        \STATE Set $x_i^{t+1}(k) \leftarrow x_i^{t,B}(k)$, for all $k \in [K]$;
   \ENDFOR
   \RETURN{} $(\bx_i^T,C_i^T)$, $i \in [m]$.
    \end{algorithmic}
\end{algorithm}

\begin{assumption}\label{asmpt:met}
The distance function $g$ is a metric, i.e., for all $x,y,z \in \R^d$:
\begin{enumerate}
    \item $g(x,y) \geq 0$ and $g(x,y) = 0$ if and only if $x = y$;
    \item $g(x,y) = g(y,x);$
    \item $g(x,y) \leq g(x,z) + g(z,y)$.
\end{enumerate}  
\end{assumption}

\begin{assumption}\label{asmpt:g&f}
The loss $f$ preserves the ordering with respect to $g$, i.e., for all $x,y,z \in \R^d$
\begin{equation*}
    f(x,y) < f(z,y) \text{ if } g(x,y) < g(z,y) \text{ and } f(x,y) = f(z,y) \text{ if } g(x,y) = g(z,y).
\end{equation*}
\end{assumption}

\begin{remark}
    Assumption \ref{asmpt:met} requires $g$ to be a well-behaved distance function, while Assumption \ref{asmpt:g&f} ensures that the cluster update step \eqref{eq:reassign} does not increase the cost $J_\rho$.
\end{remark}

\noindent \textbf{Example 5}. For $g(x,y) = \|x - y\|$ being the Euclidean distance, losses $f_1(x,y) = \nicefrac{1}{2}g(x,y)^2$, $f_2(x,y) = \phi_\delta(g(x,y))$, $f_3(x,y) = \log[1 + \exp(g(x,y)^2)]$, $f_4(x,y) = h_\eta(g(x,y))$ satisfy Assumptions \ref{asmpt:met}, \ref{asmpt:g&f} and recover Examples 1-4. For $g(x,y) = \sqrt{\langle x - y,A(x - y)}\rangle$ being a Mahalanobis distance, with $A$ positive definite, the losses $f_1$-$f_4$ again satisfy Assumptions~\ref{asmpt:met},~\ref{asmpt:g&f} and give rise to novel Mahalanobis distance distributed clustering methods.

By specializing $f$ and $g$, we get instances of DGC-$\mathcal{F}_\rho$. We now give some examples.

\textbf{DGC-KM$_\rho$}: for $g(x,y) = \|x - y\|$, $f(x,y) = \frac{1}{2}\|x - y\|^2$, we get the DGC-KM$_\rho$ algorithm, with center update equation
\begin{tcolorbox}[ams equation]
    x_i^{t,b+1}(k) =\bigg(1 - \alpha\Big[\nicefrac{1}{\rho}\sum_{r \in C^{t+1}_i(k)}w_{i,r} + |\calN_i| \Big] \bigg)x^{t,b}_i(k) + \frac{\alpha}{\rho}\sum_{r \in C_i^{t+1}(k)}w_{i,r}y_{i,r} + \alpha\sum_{j \in \calN_i} x^{t,b}_j(k).\label{eq:dgc-km}
\end{tcolorbox}

\textbf{DGC-HL$_\rho$}: for $g(x,y) = \|x - y\|$, with $f(x,y) = \phi_\delta(\|x - y\|)$, we get the DGC-HL$_\rho$ algorithm, with center update equation given by 
\begin{tcolorbox}[enhanced, colback=blue!10!white, colframe=blue!50!black]
\begin{equation}
\begin{aligned}
    x_i^{t,b+1}(k) &= \alpha\sum_{j \in \calN_i} x^{t,b}_j(k) + \frac{\alpha}{\rho}\bigg(\sum_{r \in C_{i,n}^{t+1}(k)}w_{i,r}y_{i,r} + \sum_{r \in C_{i,f}^{t+1}(k)}\frac{\delta w_{i,r}y_{i,r}}{\|x_i^{t,b}(k) - y_{i,r}\|}\bigg) \\ 
    &+ \bigg(1 - \alpha\Big[\nicefrac{1}{\rho}\sum_{r \in C_{i,n}^{t+1}(k)}w_{i,r} + \nicefrac{1}{\rho}\sum_{r \in C_{i,f}^{t+1}(k)}\frac{\delta w_{i,r}}{\|x_i^{t,b}(k) - y_{i,r}\|} + |\calN_i| \Big]\bigg) x^{t,b}_i(k),
\end{aligned}
\label{eq:dgc-hl}
\end{equation}
\end{tcolorbox}\noindent where $C_{i,n}^{t+1}(k) = \{r \in C_i^{t+1}(k): \: \|x_i^{t,b}(k) - y_{i,r}\| \leq \delta\}$ the set of points \emph{near} the current center, with $C_{i,f}^{t+1}(k) = \{r \in C_i^{t+1}(k): \: \|x_i^{t,b}(k) - y_{i,r}\| > \delta\}$ the set of points \emph{far} from the current center. 

\textbf{DGC-LL$_\rho$}: for $g(x,y) = \|x - y\|$, with $f(x,y) = \log(1 + \exp(\|x - y\|^2))$, we get the DGC-LL$_\rho$ algorithm, with the center update equation given by
\begin{tcolorbox}[enhanced, colback=blue!10!white, colframe=blue!50!black]
\begin{equation}
\begin{aligned}
    x_i^{t,b+1}(k) &= \frac{\alpha}{\rho}\sum_{r \in C_i^{t+1}(k)}\frac{2w_{i,r}y_{i,r}}{1 + \exp(-\|x_i^{t,b}(k) - y_{i,r}\|^2)} + \alpha\sum_{j \in \calN_i}\Big( x^{t,b}_j(k) - x^{t,b}_i(k)\Big) \\ &+ \bigg(1 - \nicefrac{\alpha}{\rho}\hspace{-1.2em}\sum_{r \in C^{t+1}_i(k)}\hspace{-0.25em}\frac{2w_{i,r}}{1 + \exp(-\|x_i^{t,b}(k) - y_{i,r}\|^2)} \bigg)x^{t,b}_i(k).
\end{aligned}
\label{eq:dgc-ll}
\end{equation}
\end{tcolorbox}

\textbf{DGC-FL$_\rho$}: for $g(x,y) = \|x - y\|$, with $f(x,y) = h_{\eta}(\|x-y\|)$, where we recall that the fair loss is given by $h_\eta(x) = 2\eta^2(\nicefrac{x^2}{\eta} - \log(1 + \nicefrac{x^2}{\eta}))$, where $\eta > 0$, we get the DGC-FL$_\rho$ algorithm, with the center update equation given by
\begin{tcolorbox}[enhanced, colback=blue!10!white, colframe=blue!50!black]
\begin{equation}
\begin{aligned}
    x_i^{t,b+1}(k) &= \frac{\alpha}{\rho}\sum_{r \in C_i^{t+1}(k)}\frac{4w_{i,r}\|x_i^{t,b}(k)-y_{i,r}\|^2}{1 + \nicefrac{\|x_i^{t,b}(k)-y_{i,r}\|^2}{\eta}}y_{i,r} + \alpha\sum_{j \in \calN_i} \Big(x^{t,b}_j(k) - x^{t,b}_i(k) \Big) \\ &+ \bigg(1 - \nicefrac{\alpha}{\rho}\hspace{-1em}\sum_{r \in C^{t+1}_i(k)}\hspace{-1em}\frac{4w_{i,r}\|x_i^{t,b}(k)-y_{i,r}\|^2}{1 + \nicefrac{\|x_i^{t,b}(k)-y_{i,r}\|^2}{\eta}} \bigg)x^{t,b}_i(k).
\end{aligned}
\label{eq:dgc-fl}
\end{equation}
\end{tcolorbox} We now provide some intuition behind our methods and potential applications where different methods can be deployed. If $w_{i,r} = \frac{1}{N}$ for all $i \in [m]$, $r \in [N_i]$, we can see from \eqref{eq:dgc-km} that DGC-KM$_\rho$ assigns uniform weight to each sample, while DGC-HL$_\rho$, DGC-LL$_\rho$ and DGC-FL$_\rho$ assign non-uniform weights. As such, DGC-KM$_\rho$ is well suited to applications where each sample is important and carries equal weight. For example, this is the case with data distributions for which no (or very low) presence of outliers and noisy samples is to be expected, such as light-tailed distributions, like Gaussians. On the other hand, we can see from \eqref{eq:dgc-hl} that DGC-HL$_\rho$ separates the data from cluster $C_i^{t+1}(k)$ into two groups: nearby points in $C_{i,n}^{t+1}(k)$, and faraway points in $C_{i,f}^{t+1}(k)$. Nearby points are again assigned uniform weight, with faraway points assigned the weight $\nicefrac{\delta}{\|x_i^{t,b}(k) - y_{i,r}\|} \in (0,1)$, that decays as $y_{i,r}$ gets farther away from the center. As such, DGC-HL$_\rho$ is well-suited to applications where robustness to outliers is desired, as it assigns a decreasing weight to points farther away from the center, i.e., the perceived outliers. For example, this is the case with data distributions for which moderate (or significant) presence of outliers and noisy samples is to be expected, such as heavy-tailed distributions. We further test this numerically in Section \ref{sec:num}, where we demonstrate that DGC-HL shows robustness to outliers and is well suited to applications such as outlier detection. Next, we can see from \eqref{eq:dgc-ll} that DGC-LL$_\rho$ assigns non-uniform weights to all points, given by $\nicefrac{2}{(1 + \exp(-\|x_i^{t,b}(k) - y_{i,r}\|^2))} \in [1,2)$, which increases as $y_{i,r}$ gets farther away from the center. As such, DGC-LL$_\rho$ can be seen as a fairness promoting algorithm, that aims to exploit information from the whole system by giving higher weight to faraway points, and is well-suited to applications where faraway points, i.e., outliers, carry rare and important information. Finally, we can see from \eqref{eq:dgc-fl} that DGC-FL$_\rho$ also assigns non-uniform weights to all points, given by $\nicefrac{4\|x_i^{t,b}(k)-y_{i,r}\|^2}{(1 + \nicefrac{\|x_i^{t,b}(k)-y_{i,r}\|^2}{\gamma})} \in [0,4\gamma)$, which again increases as $y_{i,r}$ gets farther away from the center. As such, DGC-FL$_\rho$ can again be seen as a fairness promoting algorithm, assigning higher weight to points further from the center, albeit in a slightly more aggressive manner than DGC-LL$_\rho$, as it can potentially disregard the points that perfectly match with the current center, by assigning them weight zero. Therefore, both DGC-LL$_\rho$ and DGC-FL$_\rho$ should be used in applications in which it is reasonable to assume that the data distribution is such that outlier points carry at least as much (or more) information as the ones concentrated around the mean. This can again be the case with heavy-tailed distributions, where we are now more interested in the outlier points.  

\begin{remark}\label{rmk:comm1}
     Communication takes place in step 8 of Algorithm \ref{alg:dist-grad-cl} and is performed $B \geq 1$ times per iteration. On the other hand, the methods in \cite{dist-clust-wsn,kar2019clustering} perform a single round of communication per iteration. As the users in all three methods exchange the same messages (namely cluster centers), it follows that our algorithm achieves the same ($B = 1$), or higher ($B > 1$) communication cost compared to \cite{dist-clust-wsn,kar2019clustering}.
\end{remark}

\begin{remark}
    Computation takes place during steps 2-5 and 7-10 of Algorithm \ref{alg:dist-grad-cl}. Cluster update (steps 2-5) requires finding closest centers for each sample, incurring the same cost as in any other center-based iterative clustering algorithm, e.g., \cite{dist-clust-wsn,pmlr-v162-armacki22a,kar2019clustering}. Center update (steps 7-10) is performed $B \geq 1$ times per iteration, with the main cost coming from evaluating the gradient of $f$ in \eqref{eq:grad_local}. On the other hand, the methods in \cite{dist-clust-wsn,kar2019clustering}, which are desinged specifically for distributed $K$-means clustering, perform a single center update per iteration. In the special case of $K$-means clustering (i.e., using squared Euclidean loss), the cost of evaluating the gradient is equal to the cost of computing cluster means, therefore, our method again either achieves the same ($B = 1$), or higher ($B > 1$) computation cost compared to \cite{dist-clust-wsn,kar2019clustering}.\footnote{For clustering beyond $K$-means, the method in \cite{kar2019clustering} is not applicable, while, by design of ADMM, it might be possible to extend the method from \cite{dist-clust-wsn} to losses beyond $K$-means, however, it would possibly involve performing an $\argmin$ step when computing the centers. This can be significantly more expensive than evaluating $B$ gradients, even with $B \gg 1$, see the discussion in Appendix \ref{app:intuition}. However, the authors in \cite{dist-clust-wsn} only provide a method for $K$-means hard clustering and do not discuss extensions, while we pursue a different, primal only gradient-based approach in our work.}
\end{remark}

\begin{remark}\label{rmk:comm2}
    Increasing $B$ results in the center update more closely approximating the $\argmin$ step, i.e., $\bx^{t+1} \approx \argmin_{\bx \in \R^{Kmd}}J_\rho(\bx,C^{t+1})$. As we show in the Appendix (see Lemmas \ref{lm:decr} and \ref{lm:final-step}), this results in faster decrease in the cost $J_\rho$ and faster convergence of centers. However, increasing $B$ incurs higher computation and communication cost. This trade-off is explored numerically in Section \ref{sec:num}. Additionally, we highlight the advantage of the gradient-based approach to naively using the $\argmin$ center update in Appendix \ref{app:intuition}.
\end{remark}

There are three algorithmic differences between DGC-$\mathcal{F}_\rho$ and centralized gradient clustering in \cite{pmlr-v162-armacki22a}. First, the center update \eqref{eq:grad_local} incorporates the consensus part, as otherwise, users would be performing the centralized gradient clustering on their local data, without collaboration. Second, users can perform multiple center updates per iteration, giving DGC-$\mathcal{F}_\rho$ more flexibility and allowing it to mimic the argmin step, akin to, e.g., Lloyd's algorithm \cite{Lloyd}. Finally, due to the consensus term, centers are updated even when the corresponding clusters are empty, which is not the case for the centralized method. These differences result in additional challenges in the convergence analysis in Section \ref{subsec:convergence}. More importantly, while the convergence guarantee in \cite{pmlr-v162-armacki22a} represents the main result, this is not the case in our work. The convergence guarantee of DGC-$\mathcal{F}_\rho$ in Section \ref{subsec:convergence} ahead shows that a stationary point of the relaxed problem \eqref{eq:general-decentr} is reached, providing no consensus guarantees, meaning that the cluster centers at different users might differ significantly. As the original distributed problem \eqref{eq:general-constr} requires centers across users to be the same, the main challenge and novelty of the paper is the consensus analysis developed in Section \ref{subsec:consensus}, where we show that our algorithm is indeed guaranteed to reach consensus and recover a solution of the original problem (2), as $\rho \rightarrow \infty$.\footnote{There are many results on clustering in federated learning, e.g., \cite{fedclust3,fedclust4,fedclust5,one_shot_clust,fedclust1,fedclust2}, however, federated algorithms in effect behave like centralized ones, keeping track of global center estimates by periodically averaging users' local centers. This significantly differs from the consensus dynamic of P2P methods, ensuring a clustering of the global data is produced by design. That is not the case in the P2P setup in our work, requiring a separate consensus analysis.} The main technical challenge of analyzing consensus in our work compared to other distributed problems, like optimization and estimation \cite{nedic-subgrad,consensus+innovation1,consensus+innovation2,kun_dgd,swenson-dsgd}, stems from the fact that center-based clustering depends on the dynamic of both a continuous (centers) and a discrete (clusters) variable, whereas classical works are typically only concerned with the dynamic of a single type of variable.\footnote{Although \cite{dist-clust-wsn,kar2019clustering} design distributed algorithms for $K$-means clustering, the algorithm in \cite{dist-clust-wsn} is built on the ADMM dynamic and does not require a separate consensus analysis, while the convergence and consensus analyses of the algorithm in \cite{kar2019clustering} rely heavily on the closed forms of center updates and fixed points. Our analysis differs in that we utilize a generic update rule and general properties of the loss $f$ to show that fixed points are well-behaved and consensus is achieved, making it applicable to a broad class of functions.} As such, we provide a novel analysis in Section \ref{subsec:consensus}, carefully exploiting the general properties of the clustering loss (via Assumption \ref{asmpt:coerc}), to show that consensus is achieved asymptotically. 

The method in \cite{kar2019clustering} is a special case of DGC-$\mathcal{F}_\rho$, designed for the squared Euclidean loss and a user, cluster and time varying step-size\footnote{The step-size used in \cite{kar2019clustering} is given by $\alpha^t_i(k) =  \alpha/(\nicefrac{|C_i^{t+1}(k)|}{\rho} + |\calN_i|)$, with $\alpha < \min_{i \in [m]}|\calN_i|/(\max_{i \in [m]}N_i/\rho + \overline{\lambda}(L))$.}. DGC-$\mathcal{F}_\rho$ is much more general, encompassing a variety of clustering algorithms outlined in this section and  allowing multiple center updates per round. Additionally, DGC-$\mathcal{F}_\rho$ is easier to implement, as it uses the same \emph{fixed step-size} for all users\footnote{Our step-size requires knowledge of $\rho$, $\beta$ and $\overline{\lambda}(L)$, see Section \ref{sec:main}. If the cost is $K$-means, with weights $w_{i,r} = 1$, it can be shown that $\beta = \max_{i \in [m]}N_i$, which is similar to $\alpha$ used in the step-size in \cite{kar2019clustering}. In that sense, the step-size in our algorithm requires shared knowledge of the same parameters as in \cite{kar2019clustering} and can be achieved by any gossip algorithm, e.g., \cite{gossip-dimakis}, at the expense of a few extra communication rounds.}. Compared to \cite{kar2019clustering}, whose analysis is tailored to the method proposed therein (i.e., a closed-form update specialized to quadratic costs), we provide a black-box approach, relying on the generic gradient-based update rule. This makes our analysis applicable to a myriad of loss functions, while also making it more challenging, and requires introducing novel concepts, such as consensus fixed points (see Definition \ref{def:consensus-fixed_pt}), as well as utilizing general properties of the loss function and its gradients, rather than the closed-form update expression.

\section{Main results}\label{sec:main}

In this section we present the main results. Section \ref{subsec:prelim} defines \emph{fixed points}, a key concept in our analysis. Section \ref{subsec:convergence} presents convergence guarantees of DGC-$\mathcal{F}_\rho$, establishing convergence of centers to a fixed point and convergence of clusters \emph{in finite time}. Section \ref{subsec:fixed-pt} presents a closed-form expressions for fixed points, when the loss $f$ is a Bregman distance. Section \ref{subsec:consensus} studies the behaviour of fixed points as $\rho \rightarrow \infty$, establishing consensus of centers and convergence of clusters to a clustering of the full data \emph{for finite $\rho$}. For ease of notation, we drop the subscript in gradients, e.g., $\nabla J(\bx,C) \equiv \nabla_\bx J(\bx,C)$. The proofs from this section can be found in Appendix \ref{app:proofs}.

\subsection{Setting up the analysis}\label{subsec:prelim}

In this section we define key concepts used in the analysis.

\begin{definition}\label{def:U}
Let $\bx \in \R^{Kmd}$ be cluster centers. We say that $U_\bx \subset \mathcal{C}_{m,K,\D}$ is the set of optimal clusterings with respect to $\bx$, if for all clusterings $C \in U_\bx$, \eqref{eq:reassign} is satisfied.
\end{definition}

\begin{definition}\label{def:fix-pt}
The pair $(\bx^\star,C^\star)$ is a fixed point of DGC-$\mathcal{F}_\rho$, if: 
\begin{enumerate}
    \item $C^\star \in U_{\bx}$; 
    \item $\nabla J_\rho(\bx^\star,C^\star) = 0$.
\end{enumerate}
\end{definition}

\begin{remark}\label{rmk:stat}
    Definition \ref{def:fix-pt} requires $(\bx^\star,C^\star)$ to be a stationary point of $J_\rho$, in the sense that clusters $C^\star$ are optimal with respect to centers $\bx^\star$ and centers $\bx^\star$ are optimal when clusters $C^\star$ are fixed. As such, it is not possible to further improve the clustering, nor the centers at a fixed point. Recalling that $J_\rho$ is non-convex, reaching a fixed point is the best we can do.   
\end{remark}

\begin{remark}
    Definition \ref{def:fix-pt} provides a general definition of a fixed point with respect to our algorithm DGC-$\mathcal{F}_\rho$. While it includes many possible fixed points, we show that in practice, DGC-$\mathcal{F}_\rho$ only reaches a well-behaved subset of the set of fixed points satisfying Definition \ref{def:fix-pt} (see Lemma \ref{lm:bdd-fix-pt} in Appendix \ref{subsec:proofs3}).
\end{remark}

\begin{definition}\label{def:Ubar}
The set $\overline{U}_\bx \subset \mathcal{C}_{m,K,\D}$ is the set of clusterings, such that: 
\begin{enumerate}
    \item $\overline{U}_\bx \subseteq U_\bx$; 
    \item $\nabla J_\rho(\bx,C) = 0$, for all $ C \in \overline{U}_\bx$.
\end{enumerate} 
\end{definition}

By Definitions~\ref{def:fix-pt} and~\ref{def:Ubar}, $\bx$ can be a fixed point if and only if $\overline{U}_{\bx} \neq \emptyset$. As such, we will call $\bx$ a fixed point if $\overline{U}_{\bx} \neq \emptyset$.

\subsection{Convergence to fixed points}\label{subsec:convergence}

In this section we show that centers produced by DGC-$\mathcal{F}_\rho$ are guaranteed to converge to a fixed point. Recalling Remark~\ref{rmk:stat}, this is, in general, the best one can achieve. 

\begin{theorem}\label{thm:convergence}
    Let Assumptions~\ref{asmpt:data}-\ref{asmpt:g&f} hold. For the step-size $\alpha < 1/(\nicefrac{\beta}{\rho} + \lambda_{\max}(L))$, any initialization $\bx^0 \in \R^{Kmd}$ and $\rho \geq 1$, the sequence of centers $\{\bx^t\}_{t \in \N}$ generated by DGC-$\mathcal{F}_\rho$ converges to a fixed point, i.e., a $\bx^\star \in \R^{Kmd}$, such that $\overline{U}_{\bx^\star} \neq \emptyset$. Moreover, the clusters converge in finite time, i.e., there exists a $t_0 > 0$, such that $U_{\bx^t} \subset U_{\bx^\star}$, for all $t \geq t_0$. 
\end{theorem}

In order to prove Theorem \ref{thm:convergence}, a series of lemmas are introduced, similarly to the approach in \cite{pmlr-v162-armacki22a}. However, due to the problem being distributed and the algorithmic differences discussed in Section \ref{sec:methods}, the majority of proofs from \cite{pmlr-v162-armacki22a} are not directly applicable and careful modifications of their arguments is needed. The lemmas can be found in Appendix \ref{app:proofs}. Some remarks are now in order.

\begin{remark}
    The condition on the step-size in Theorem \ref{thm:convergence} is typical of iterative optimization methods using a fixed step-size, expressed in terms of problem related constants, such as the smoothness and penalty parameters $\beta$ and $\rho$, as well as the largest eigenvalue of graph Laplacian matrix, $\lambda_{\max}(L)$, see \cite{kun_dgd,xin-fast,lectures_on_cvxopt} and references therein.
\end{remark}

\begin{remark}
    Theorem \ref{thm:convergence} states that the sequence of centers generated by DGC-$\mathcal{F}_\rho$ is guaranteed to converge to a fixed point, for any center initialization. This is a strong result, which provides great flexibility, in the sense that no synchronization of users' initial centers is needed beforehand, while also allowing for initialization schemes such as distributed $K$-means$++$ \cite{dist-soft} to be implemented.
\end{remark}

\begin{remark}
    While the convergence of centers in Theorem \ref{thm:convergence} is asymptotic in nature, it still guarantees that clusters converge in finite time, which is of great importance, as in practice, we are often interested only in the clustering of the data. In that sense, Theorem \ref{thm:convergence} guarantees that \emph{a solution will be provided in finite time}.   
\end{remark}

\subsection{Fixed point analysis - Bregman clustering}\label{subsec:fixed-pt}

In this section we provide closed-form expressions for fixed points of DGC-$\mathcal{F}_\rho$, when the loss is Bregman, i.e., $f(x,y) = \psi(y) - \psi(x) - \langle \nabla \psi(x), y - x \rangle$, with $\psi$ strictly convex. To apply DGC-$\mathcal{F}_\rho$, we require a distance $g$ satisfying Assumptions \ref{asmpt:met}, \ref{asmpt:g&f}. As noted in \cite{pmlr-v162-armacki22a}, it is possible to find such $g$ for many Bregman distances, via \cite{chen_bregman_metrics,bregman_triangle}, which state that a class of Bregman distances are squares of distance metris. Some examples of such Bregman distances are given in Appendix \ref{app:proofs}. We make the following assumption on the Bregman inducing function $\psi$.

\begin{assumption}\label{asmpt:psi}
    $\psi$ is strictly convex, twice differentiable and induces a Bregman distance that is a square of a metric. 
\end{assumption}

\begin{lemma}\label{lm:fix-pt}
    Let $f(x,y)$ be a Bregman distance, with $\psi$ satisfying Assumption \ref{asmpt:psi}. Then, any fixed point $(\bx^\star,C^\star)$ of DGC-$\mathcal{F}_\rho$ satisfies, for all $i\in[m], \: k\in[K]$
    \begin{equation*}
        x^\star_i(k) = P_{i,k}^{-1}\Big(\nicefrac{1}{\rho}\nabla^2\psi(x_i^\star(k))\hspace{-0.6em}\sum_{r \in C_i^\star(k)}\hspace{-0.6em}w_{i,r}y_{i,r} + \sum_{j \in \calN_i}x_j^\star(k)\Big),
    \end{equation*} where $P_{i,k} = \frac{1}{\rho}\nabla^2\psi(x_i^\star(k))\sum_{r \in C_i^\star(k)}w_{i,r} + |\calN_i|I_{d}$.
\end{lemma}

\begin{remark}
    If $f(x,y) = \frac{1}{2}\|x - y\|^2$ (i.e., $K$-means clustering), it can be shown that $\psi(x) = \frac{1}{2}\|x\|^2$, e.g., \cite{JMLR:v6:banerjee05b}. In that case $\nabla^2 \psi(x) \equiv I_{d}$, hence fixed points of DGC-KM satisfy, for any $i \in [m]$, $k \in [K]$
    \begin{equation*}
        x^\star_i(k) = \frac{\nicefrac{1}{\rho}\sum_{r \in C_i^\star(k)}w_{i,r}y_{i,r} + \sum_{j \in \calN_i}x_j^\star(k)}{\nicefrac{1}{\rho}\sum_{r \in C_i^\star(k)}w_{i,r} + |\calN_i|}.
    \end{equation*} This is consistent with the fixed points of the method from \cite{kar2019clustering}, termed \emph{generalized Lloyd minima}. 
\end{remark}

\subsection{Consensus and fixed points}\label{subsec:consensus}

In this section we study the behaviour of the sequences of fixed points generated by DGC-$\mathcal{F}_\rho$. As discussed in Section \ref{sec:methods}, convergence guarantees alone do not suffice in showing that a solution of the original problem \eqref{eq:general-constr} will be produced. Noting that fixed points of DGC-$\mathcal{F}_\rho$ depend on $\rho$, we are then interested in the behaviour of the family of fixed points $\{\bx_\rho \}_{\rho \geq 1}$ of DGC-$\mathcal{F}_\rho$, as $\rho \rightarrow \infty$. To start, we state the assumption used in this section. 

\begin{assumption}\label{asmpt:dist-norm}
    The distance function $g$ is the Euclidean distance and the gradient of $f$ is of the form $\nabla f(x,y) = \gamma(x - y)$, where $\gamma \equiv \gamma(x,y) \geq 0$ may depend on $x$ and $y$. 
\end{assumption}

We show in Appendix \ref{app:proofs} that Assumption \ref{asmpt:dist-norm} is satisfied for all of $K$-means, Huber, Fair and Logistic losses. Next, we define the key concept in this section, that of \emph{consensus fixed points}.

\begin{definition}\label{def:consensus-fixed_pt}
    The point $\overline{\bx} \in \R^{Kmd}$ is a consensus fixed point if: 
    \begin{enumerate}
        \item $\overline{\bx}_1 = \ldots = \overline{\bx}_m$;
        \item there exists a $\overline{C} \in U_{\overline{\bx}},$ such that $\mathbf{1}^\top \nabla J(\overline{\bx},\overline{C}) = 0$, where $\mathbf{1} = 1_m \otimes I_{Kd}$.
    \end{enumerate} 
\end{definition}

\begin{remark}
    Consensus fixed points are not fixed points per Definition~\ref{def:fix-pt}, as they may not satisfy $\nabla J(\overline{\bx},\overline{C}) = 0$.
\end{remark}

We can write any consensus fixed point as $\overline{\bx} = 1_m \otimes \bx$, where $\bx \in \R^{Kd}$ is the vector stacking of $K$ centers. From point 2) in Definition \ref{def:consensus-fixed_pt} and the definition of $\nabla J$, we get
\begin{equation}\label{eq:zero-grad}
    \mathbf{1}^\top\nabla J(\overline{\bx},\overline{C}) = \sum_{i \in [m]}\nabla H(\bx,\overline{C}_i) = 0.
\end{equation} If we define $C = (C(1),\ldots,C(K))$, where $C(k) = \cup_{i \in [m]}\overline{C}_i(k)$, it follows from~\eqref{eq:zero-grad} that, for all $k \in [K]$
\begin{align}
    \sum_{r \in C(k)}\hspace{-0.6em}w_{i,r}\nabla f(x(k),y_{i,r}) &= \sum_{i \in [m]}\sum_{r \in \overline{C}_i(k)}\hspace{-0.6em}w_{i,r}\nabla f(x(k),y_{i,r}) = [\mathbf{1}^\top \nabla J(\overline{\bx},\overline{C})]_k = 0, \label{eq:zero-grad2}
\end{align} i.e., $(\bx,C)$ is a stationary point of the centralized problem \eqref{eq:general-clust}. The converse holds as well, i.e., any stationary point $(\bx,C)$ of \eqref{eq:general-clust} induces a consensus fixed point $(\overline{\bx},\overline{C})$, where $\overline{\bx} = 1_m \otimes \bx$ and $\overline{C} = (\overline{C}_1,\ldots,\overline{C}_m)$, with $\overline{C}_i(k) = C(k) \cap \D_i$. As such, our aim is to find consensus fixed points, ensuring that a clustering of the full data is produced. To provide further intuition on consensus fixed points, it is instructive to consider the special case of Bregman loss functions.

\begin{lemma}\label{lm:bregman-cons-fp}
    If the clustering loss $f$ is a Bregman distance, then any consensus fixed point $(\overline{\bx},\overline{C})$, with $\overline{\bx} = 1_m \otimes \bx$, corresponds to a \emph{Lloyd point}, i.e., for all $k \in [K]$
    \begin{equation*}
        x(k) = \frac{1}{W_k}\sum_{i \in [m]}\sum_{r \in \overline{C}_i(k)}w_{i,r}y_{i,r} = \frac{1}{W_k}\sum_{r \in C(k)}w_{i,r}y_{i,r},
    \end{equation*} where $W_k = \sum_{i \in [m],r \in \overline{C}_i(k)}w_{i,r}$ and $C(k) = \cup_{i \in [m]}\overline{C}_i(k)$.
\end{lemma}

Lemma \ref{lm:bregman-cons-fp} shows that in the case of Bregman losses, consensus fixed points correspond to Lloyd-type points, i.e., cluster means. The next result characterizes consensus fixed points when their corresponding clusters are non-empty.

\begin{lemma}\label{lm:cons-fp-D}
    Let Assumption \ref{asmpt:dist-norm} hold and let $(\overline{\bx},\overline{C})$ be a consensus fixed point, i.e., $\overline{\bx} = \mathbf{1} \otimes \bx$. Then, $x(k) \in \overline{co}(\D)$, for all $k \in [K]$ for which the $k$-th cluster of at least one user is non-empty, i.e., $\cup_{i \in [m]}\overline{C}_i(k) \neq \emptyset$.
\end{lemma}

Note that centers whose corresponding clusters are empty do not contribute to the cost or the gradient. As such, they can be freely assigned and we can always choose a center that belongs to $\overline{co}(\D)$. Define $\mathcal{X} = \big\{\overline{\bx} \in \R^{Kmd}: \: \overline{\bx} = 1_m \otimes \bx \text{ is a consensus fixed point and } x(k) \in \overline{co}(\D), \: \text{for all } k \in [K]\big\}$. It then follows from Lemma \ref{lm:cons-fp-D} and preceding discussion that $\mathcal{X} \neq \emptyset$. We now state the main result of this section.

\begin{theorem}\label{thm:cons-conv}
    Let Assumption \ref{asmpt:dist-norm} hold and $\{\bx_\rho\}_{\rho \geq 1}$ be a family of fixed points of DGC-$\mathcal{F}_\rho$, for all $\rho \geq 1$ and fixed center initialization $\bx^0 \in \R^{Kmd}$.\footnote{Note that, for each fixed $\rho \geq 1$, the fixed point $\bx_\rho$ can be any fixed point generated by our method DGC-$\mathcal{F}_\rho$, when initialized using $\bx^0$.} Then, the sequence converges to the set $\mathcal{X}$ with a rate $\mathcal{O}(\nicefrac{1}{\rho})$, i.e., $d(\bx_\rho,\mathcal{X}) = \mathcal{O}(\nicefrac{1}{\rho})$, where $d(\bx,\mathcal{X}) = \inf_{\by \in \mathcal{X}}\|\bx - \by\|$. Moreover, the clusters converge for finite $\rho$, i.e., there exists a $\rho_0 \geq 1$, such that $U_{\bx_\rho} \subseteq U_{\overline{\bx}}$, for all $\rho \geq \rho_0$ and some $\overline{\bx} \in \mathcal{X}$.
\end{theorem}

Theorem \ref{thm:cons-conv} states that consensus is achieved at rate $\mathcal{O}\left(\nicefrac{1}{\rho}\right)$, which we verify numerically in Section \ref{sec:num}. Moreover, clusters converge for a finite $\rho$, implying that \emph{we are guaranteed to produce a clustering of the full data for a finite value of $\rho$}. Together with Lemma \ref{lm:bregman-cons-fp}, Theorem \ref{thm:cons-conv} implies that, for the special case of Bregman losses, DGC-$\mathcal{F}_\rho$ is guaranteed to converge to the set of Lloyd points, which is consistent with centralized methods with Bregman losses, e.g., \cite{JMLR:v6:banerjee05b,pmlr-v162-armacki22a}. Moreover, the parameter $\rho$ offers an inherent trade-off, in that, while higher values of $\rho$ guarantee consensus, increasing $\rho$ results in slower convergence of DGC-$\mathcal{F}_\rho$, as it takes more time to optimize the clustering part of the cost, with similar observations made in \cite{kar2019clustering}. As such, $\rho$ offers a trade-off between convergence speed and producing a clustering of the full data. In applications where convergence speed is paramount, moderate values of $\rho$ are apt, otherwise Theorem \ref{thm:cons-conv} implies choosing $\rho$ and $t$ sufficiently large to guarantee DGC-$\mathcal{F}_\rho$ produces a clustering of the full data.

\begin{remark}
    Theorem \ref{thm:cons-conv} guarantees that a clustering corresponding to a consensus fixed point can be attained for finite values of $\rho$, i.e., values satisfying $\rho \geq \rho_0$, for some finite $\rho_0 \geq 1$. The specific value of $\rho_0$ depends on the choice of loss function $f$, users' datasets $\D_i$, as well as center initializations $\bx_i^0$, $i \in [m]$. As such, in general, it is not possible to provide an exact value of $\rho_0$, beyond its existence. However, we provide a further discussion on approximating $\rho_0$ in Appendix \ref{app:intuition}.  
\end{remark}

\begin{remark}\label{rmk:effect-of-m}
    Varying the number of users over a fixed dataset (i.e., without adding new samples when introducing new users), does not impact the asymptotic accuracy achieved by our algorithm.\footnote{By asymptotic accuracy we mean the accuracy achieved by a consensus fixed point reached by our algorithm, evaluated over the entire, joint dataset.} However, increasing the number of users can impact the speed at which the asymptotic accuracy is achieved, by affecting the network connectivity and requiring more iterations to ensure convergence of our method. We verify this behaviour numerically in Appendix \ref{app:exp}.
\end{remark}

\begin{remark}
    Specializing~\eqref{eq:zero-grad2} to $K$-means cost, any consensus fixed point $(\overline{\bx},\overline{C})$, with $\overline{\bx} = 1_m \otimes \bx$ satisfies
    \begin{equation*}
        x(k) = \frac{1}{W_k}\sum_{i \in [m]}\sum_{r \in \overline{C}_i(k)}w_{i,r}y_{i,r} = \frac{1}{W_k}\sum_{r \in C(k)}w_{i,r}y_{i,r},
    \end{equation*} where $W_k = \sum_{i \in [m]}\sum_{r \in \overline{C}_i(k)}w_{i,r}$ and $C(k) = \cup_{i \in [m]}\overline{C}_i(k)$, i.e., consensus fixed points correspond to means of global clusters, also known as \emph{Lloyd points}.
\end{remark}

\section{Numerical results}\label{sec:num}

In this section we present numerical results. All experiments were implemented in python, averaged across 5 runs. We use uniform weights, i.e., $w_{i,r} = \frac{1}{N}$, for all $r \in \D_i$ and $i \in [m]$. For Huber and fair losses, we set $\delta = 5$, $\gamma = 1$, unless specified otherwise. We conduct experiments on synthetic and real data, including Iris \cite{iris}, MNIST \cite{lecun-mnist} and CIFAR10 \cite{krizhevsky2009learning}. In particular, we use the full Iris data, a subset of MNIST consisting of the first seven digits (MNIST7), subsets of CIFAR10 consisting of the first three (CIFAR3) and eight (CIFAR8) classes, as well as full MNIST and CIFAR10 datasets. We consider two setups with respect to data distribution across users. In the first, dubbed \emph{homogeneous data}, users have access to data from all underlying clusters, in equal proportion. In the second, dubbed \emph{heterogeneous data}, users only have access to a subset of underlying clusters, possibly in varying proportion.\footnote{Meaning that different users can have access to different number of clusters, as well as different number of samples per cluster.} For a detailed description of the data, as well as additional extensive experiments, see Appendix \ref{app:exp}. 

\begin{figure}[htp]
\centering
\begin{tabular}{ll}
\includegraphics[width=0.4\columnwidth]{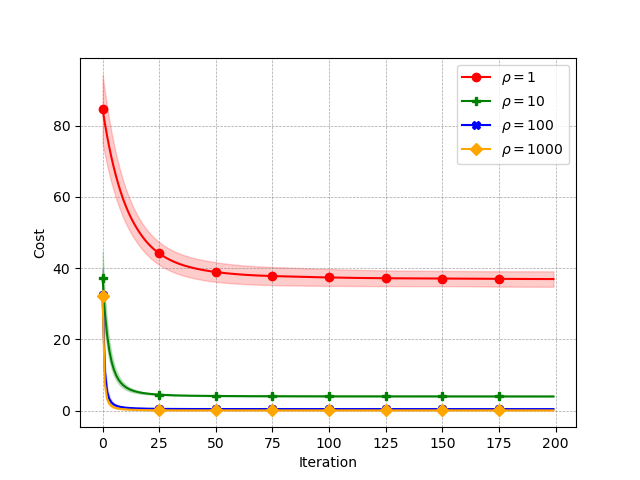}
&
\includegraphics[width=0.4\columnwidth]{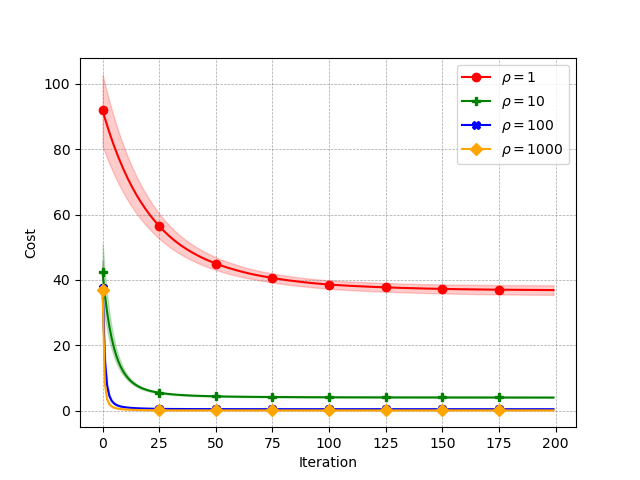}
\\
\includegraphics[width=0.4\columnwidth]{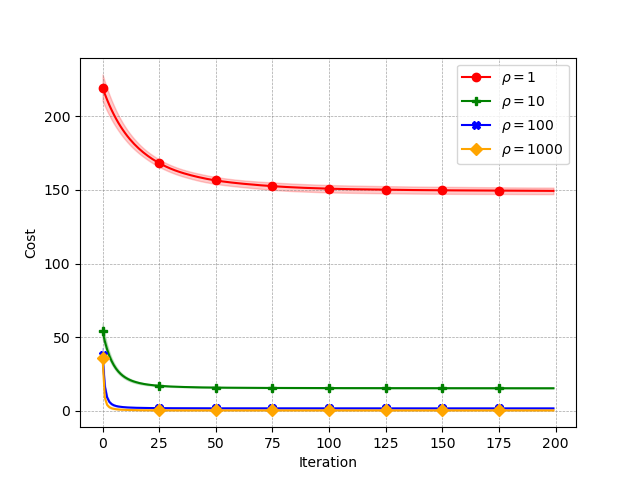}
&
\includegraphics[width=0.4\columnwidth]{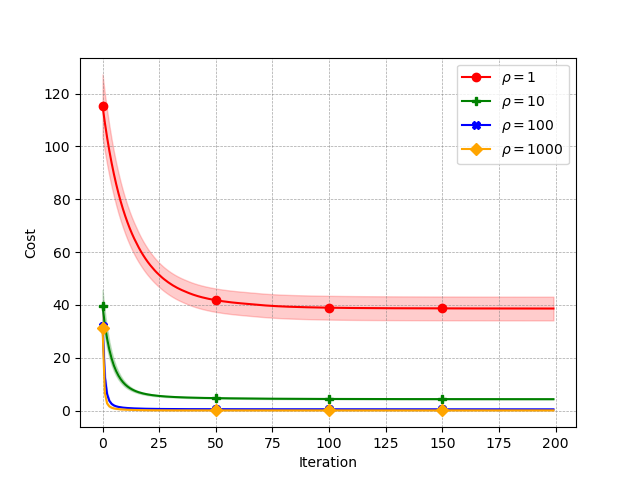}
\end{tabular}
\caption{Behaviour of $J_\rho$ for different $\rho$ and $B = 1$. Left to right: DGC-KM$_\rho$, DGC-HL$_\rho$ upper and DGC-LL$_\rho$, DGC-FL$_\rho$ lower row.}
\label{fig:iris-rho-cost}
\end{figure}

\begin{table}[htp]
\caption{Effect of $\rho$ on maximum center distance after $T = 500$ iterations.}
\label{tab:rho}
\begin{center}
\begin{small}
\begin{sc}
\begin{tabular}{lcccc}
\toprule
 & $\rho = 1$ & $\rho = 10$ & $\rho = 10^2$ & $\rho = 10^{3}$ \\
\midrule
DGC-KM$_\rho$ & $1.16$ & $3.3 \times 10^{-1}$ & $4.7 \times 10^{-2}$ & $5 \times 10^{-3}$ \\
DGC-HL$_\rho$ & $1.17$ & $3.1 \times 10^{-1}$ & $4.8 \times 10^{-2}$ & $5 \times 10^{-3}$ \\
DGC-LL$_\rho$ & $1.23$ & $4.3 \times 10^{-1}$ & $6.1 \times 10^{-2}$ & $8 \times 10^{-3}$ \\
DGC-FL$_\rho$ & $1.64$ & $4.4 \times 10^{-1}$ & $7.7 \times 10^{-2}$ & $10\times 10^{-3}$ \\
\bottomrule
\end{tabular}
\end{sc}
\end{small}
\end{center}
\vskip -0.1in
\end{table}

The first set of experiments, using Iris data, aims to verify our theory. We consider a network of $m = 10$ users, communicating over a ring graph, with homogeneous data distribution. Users initialize their centers by choosing a random sample from each class\footnote{By ``class'' here we mean the true underlying cluster. As Iris, MNIST and CIFAR10 are labeled datasets, we have knowledge of the true classes/clusters.} of their local data. We evaluate the cost $J_\rho$ and consensus for $B = 1$ and varying values of $\rho$. Consensus is measured via maximum center distance among users, i.e., $\max_{i,j \in [m]}\|\bx_i - \bx_j\|$.  The results are presented in Figure \ref{fig:iris-rho-cost} and Table \ref{tab:rho}. The solid lines in Figure \ref{fig:iris-rho-cost} represent the average performance, with the shaded regions showing standard deviation. As predicted in Lemma \ref{lm:decr}, the cost is decreasing in each iteration (Figure \ref{fig:iris-rho-cost}). Table \ref{tab:rho} presents the maximum center distance across users after $T = 500$ iterations, with $B=1$. As predicted in Theorem \ref{thm:cons-conv}, the maximum center distance is of order $\mathcal{O}\left(\nicefrac{1}{\rho} \right)$.

\begin{table}[htp]
\caption{Robustness to initialization.}
\label{tab:init}
\begin{center}
\begin{small}
\begin{sc}
\begin{tabular}{lcc}
\toprule
 & Warm start & Random \\
\midrule
 SKL-KM & $89.3 \pm 0.0 \%$ & $81.7 \pm 14.5 \%$  \\ 
CGC-KM & $89.1 \pm 0.3 \%$ & $81.6 \pm 14.8 \%$ \\
LGC-KM & $89.3 \pm 2.9 \%$ & $82.5 \pm 3.2 \%$ \\
DGC-KM$_\rho$ & $90.6 \pm 0.5 \%$ & $\mathbf{90.5 \pm 3.6 \%}$ \\
DGC-HL$_\rho$ & $90.6 \pm 0.5 \%$ & $89.3 \pm 6.1 \%$ \\
DGC-LL$_\rho$ & $\mathbf{90.7 \pm 0.2 \%}$ & $86.7 \pm 5.5 \%$ \\
ADMM-KM & $88.7 \pm 0.2 \%$ &  $88.6 \pm 0.0 \%$ \\
\bottomrule
\end{tabular}
\end{sc}
\end{small}
\end{center}
\vskip -0.1in
\end{table}

The second set of experiments, using the same network and data setup as the previous, aims to test robustness to initialization. It was observed in \cite{dist-clust-wsn} that their distributed ADMM-based $K$-means method (ADMM-KM) is more robust to initialization than centralized $K$-means. The intuition behind this phenomena stems from the fact that, while a centralized algorithm has only one initialization, distributed algorithms in effect have $m$ initializations, one for each user, with the consensus dynamic moving local centers toward a joint solution, negating the possible effects of bad initialization at some users and allowing distributed algorithms to reach a solution of better quality. To further test this phenomena, we perform two sets of experiments. In the first, centers for all methods are initialized by choosing a random sample from each class, while in the second, centers are initialized uniformly at random.\footnote{For distributed methods, random sampling is done locally at each user, i.e., each user draws random centers from their own local dataset.}As such, methods in the first experiment \emph{exploit knowledge of the underlying clustering structure} and we refer to this setup as \emph{warm start}, while in the second, methods \emph{are oblivious to the underlying clustering structure} and we refer to this setup as \emph{random}.

\begin{table*}
\caption{Clustering accuracy on homogeneous data. We use $B = 1$ and $\rho = 10$ for MNIST7, $\rho = 100$ for CIFAR3 and CIFAR8, with $\rho = 1000$ on full MNIST and CIFAR10 data.}
\label{tab:hom}
\begin{adjustwidth}{-1in}{-1in}
\begin{center}
\begin{small}
\begin{sc}
\begin{tabular}{lccccccc}
\toprule
 & SKL-KM & CGC-KM & LGC-KM & DGC-KM$_\rho$ & DGC-HL$_\rho$ & DGC-LL$_\rho$ & ADMM-KM \\
\midrule
MNIST7 & $73.7 \pm 0.2 \%$ & $73.6 \pm 0.5\%$ & $62.9 \pm 3.1\%$ & $73.3 \pm 1.1 \%$ & $\mathbf{74.4 \pm 1.3\%}$ & $70.9 \pm 1.6\%$ & $73.9 \pm 0.2\%$ \\
CIFAR3 & $50.6 \pm 0.4\%$ & $50.1 \pm 0.4\%$ & $41.4 \pm 4.4\%$ & $\mathbf{51.2 \pm 1.9\%}$ & $49.1 \pm 4.2\%$ & $46.0 \pm 3.6\%$ & $50.6 \pm 0.7\%$\\
CIFAR8 & $21.2 \pm 0.1\%$ & $21.4 \pm 0.3\%$ & $16.7 \pm 0.5\%$ & $20.9 \pm 0.4\%$ & $\mathbf{21.5 \pm 0.9\%}$ & $20.6 \pm 0.5\%$ & $20.9 \pm 0.2\%$ \\
MNIST & $51.8 \pm 0.8 \%$ & $52.7 \pm 3.2\%$ & $18.79 \pm 1.4\%$ & $52.5 \pm 3.5 \%$ & $45.8 \pm 1.3\%$ & $43.3 \pm 4.6\%$ & $\mathbf{53.2 \pm 1.6\%}$\\
 CIFAR10 & $20.4 \pm 0.3 \%$ & $\mathbf{20.5 \pm 0.7 \%}$ & $12.7 \pm 0.0 \%$ & $19.6 \pm 0.6 \%$ & $19.6 \pm 0.2 \%$ & $20.1 \pm 1.1 \%$ & $17.9 \pm 0.0 \%$\\
\bottomrule
\end{tabular}
\end{sc}
\end{small}
\end{center}
\end{adjustwidth}
\vskip -0.1in
\end{table*}

We evaluate the performance of our methods DGC-KM, DGC-HL and DGC-LL with $B = 1$, benchmarked against centralized gradient clustering with $K$-means (CGC-KM) \cite{pmlr-v162-armacki22a}, scikit-learn's state-of-the-art centralized $K$-means (SKL-KM) implementation \cite{scikit-learn}, distributed ADMM-KM \cite{dist-clust-wsn}, and a local $K$-means clustering method (LGC-KM), where users cluster their data in isolation, using gradient-based clustering from \cite{pmlr-v162-armacki22a}. We set $\rho = 10$ for ADMM-KM and our methods. All the methods are run for $T = 1.000$ iterations. We measure the clustering quality by comparing clustering produced labels with true labels, i.e., \emph{accuracy}.\footnote{We account for possible label permutation, by taking the highest possible accuracy with respect to all label permutations. Note that accuracy achieved by unsupervised methods presented in our results should not be compared to the accuracy achieved by the more powerful supervised learning methods.} For DGC, LGC and ADMM, we report the average accuracy across users. The SKL-KM method uses the more powerful $K$-means++ initialization for the warm start case, while it is assigned the same random initialization as CGC for the random case. All the distributed methods are assigned the same initialization. The results are presented in Table \ref{tab:init}. We can clearly see that distributed methods are much more resilient to initialization, with the performance of all distributed methods remaining virtually the same in both setups, while the performance of centralized algorithms deteriorates significantly under random initialization. The strong performance of LGC relative to CGC and SKL can be explained by the fact that the data is homogeneous, with each user having access to all classes, while clustering much smaller datasets. All the distributed algorithms perform on par with LGC in the warm start setup and outperform it in the random initialization setup, further highlighting how user collaboration can correct the effects of bad initialization and lead to a solution of higher accuracy.

Next, we aim to test the applicability of our framework in scenarios such as outlier detection. We consider a noisy version of the Iris dataset, where we randomly select $20\%$ of samples from each class and add Gaussian noise to each component of the selected samples, with mean $\mu = 11$ and variance $\sigma^2 = 1$, to create a fourth class of outlier points. We refer to the noisy points as ``outliers'', with the noiseless points being ``good'' points. We then spread the entire noisy Iris dataset across $m = 5$ users communicating over a ring network, with each user getting $24$ good data points ($8$ per class) and $6$ outliers. We consider the performance of our methods with $K$-means and Huber losses, DGC-KM and DGC-HL, as well as the distributed $K$-means based ADMM-KM method from \cite{dist-clust-wsn}. We set $\rho = 10$ for all three methods and $B = 1$ for our two methods, with the Huber loss parameter set to $\delta = 0.05$ for our DGC-HL method. We initialize the $K 
= 3$ centers for all three methods by first computing the true means of all good points (global, joint means), with each user locally adding Gaussian white noise ($\mu = 0$ and $\sigma^2 = 1$) to each component of the true means. We run all three methods for $T = 1.000$ iterations, after which we obtain the final centers. The final centers are averaged across $10$ runs, to account for possible perturbations in center initialization. Finally, after the averaged centers are obtained, we visualize the results, by applying the t-distributed stochastic neighbor embedding (t-SNE) method \cite{JMLR:v9:vandermaaten08a} on the entire noisy Iris data and the averaged centers produced by our two methods and ADMM-KM. The results are presented in Figure \ref{fig:noisy} below. We can see that our DGC-HL method correctly identifies all three clusters of good points, ignoring the outliers. On the other hand, both $K$-means based methods, DGC-KM and ADMM-KM, incorrectly identify the cluster of outliers as one of the three clusters, showing poor resilience to outliers, even under very strong initialization. As such, our DGC-HL method can be used for outlier detection, by marking the points furthest away from the true centers as outliers, further underlining the importance of methods beyond $K$-means and the potential of our general framework for a wide range of applications. 

\begin{figure}[htp]
\centering
\includegraphics[width=0.6\columnwidth]{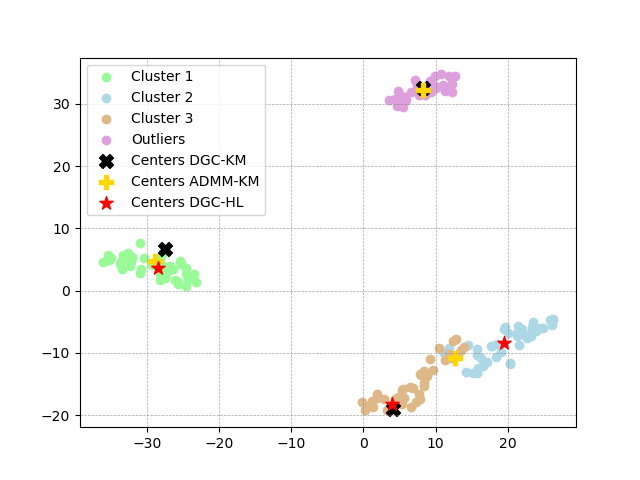}
\caption{Performance of DGC-KM, DGC-HL and ADMM-KM on noisy Iris data. We can see that DGC-HL successfully identifies the true clusters, while both DGC-KM and ADMM-KM incorrectly identify the cluster of outliers as one of the true clusters.}
\label{fig:noisy}
\end{figure}

To test the scalability of our methods on huge-scale datasets, we perform an experiment on a synthetic dataset. The data is generated by drawing $500.000$ samples from $K = 4$ different Gaussian distributions with means $\mu = \{\begin{bmatrix} 1, 1 \end{bmatrix}, \begin{bmatrix} -1, 1 \end{bmatrix}, \begin{bmatrix} 1, -1 \end{bmatrix}, \begin{bmatrix} -1, -1 \end{bmatrix} \}$ and identity covariance matrix, for a total of $2.000.000$ samples ($40\times$ larger than the largest real datasets used, MNIST and CIFAR10). We again consider a ring network of $m = 10$ users, with homogeneous data. We benchmark the performance of DGC-KM with $B = 1$ against SKL-KM, CGC-KM and ADMM-KM. Centers of all methods are initialized by randomly choosing $K = 4$ samples from the data, oblivious to the underlying clusters. We set $\rho = 100$ for both DGC-KM and ADMM-KM. The iterative algorithms were run for $T = 1.000$ iteartions, while SKL-KM terminated early, due to an inbuilt termination criteria. We report the number of iterations, time per iteration (in seconds) and accuracy. The results are presented in Table \ref{tab:huge-data}. We can see that our method performs on par with centralized methods in terms of accuracy, while maintaining a similar time per iteration, demonstrating that it scales well to huge datasets.\footnote{The slight slowdown of time per iteration of distributed methods compared to centralized ones can be explained by the fact that distributed methods were deployed in a ``simulated'' distributed environment, where the execution was done sequentially, instead of executing in parallel, optimally. Additionally, in contrast to an optimized software library like scikit-learn, our implementation of the various distributed algorithms is not optimized.}  

\begin{table}[htp]
\caption{Performance of clustering methods on huge-scale data.}
\label{tab:huge-data}
\begin{center}
\begin{threeparttable}
\begin{small}
\begin{sc}
\begin{tabular}{lccc}
\toprule
 & Num. iter. & Time per iter. & Acc. ($\%$) \\
\midrule
SKL-KM & $79 \pm 7$ & $0.025$ & $70.8 \pm 0.001$ \\
CGC-KM & $1000$ & $0.066$ & $70.8 \pm 0.001$ \\
DGC-KM & $1000$ & $0.093$ & $70.3 \pm 0.9$ \\
ADMM-KM & $1000$ & $0.099$ & $67.6 \pm 1.2$ \\
\bottomrule
\end{tabular}
\end{sc}
\end{small}
\end{threeparttable}
\end{center}
\vskip -0.1in
\end{table}

Finally, we test the performance of our methods on real datasets, with both homogeneous and heterogeneous data.\footnote{Note that this distinction is irrelevant for centralized methods, as they have access to the entire data in both cases.} In the homogeneous setup we again consider $m = 10$ users on a ring graph and measure performance via accuracy, testing the performance on Iris, MNIST3, CIFAR3-8, as well as the full MNIST and CIFAR10 datasets. In the heterogeneous setup we use a ring graph of $m = 15$ users for Iris data, with each user having access to only two classes. For the MNIST7 and CIFAR3 data in heterogeneous setup, we consider a network of $m = 10$ users, with users communicating over a Erdos-Renyi graph with connectivity parameter $p = 0.5$. For MNIST7, each user has access to data from at least three and at most five of the underlying seven classes, while for CIFAR3, users have access to two of the underlying three classes. Due to the difficulty of tracking label permutations for heterogeneous data, we evaluate the performance using \emph{Adjusted Rand Index} (ARI) score between the true labels and the ones produced by clustering. The ARI score measures alignment of labels between two clusterings, while accounting for possible label permutations. Its value lies in $[0,1]$, with higher values meaning better label alignment. In both settings we test the performance of our methods DGC-KM, DGC-HL and DGC-LL with SKL-KM, CGC-KM and ADMM-KM. For our methods we set $B = 1$, with $\rho$ the same for our methods and ADMM-KM and varying for different datasets (see Tables \ref{tab:hom} and \ref{tab:het}). We initialize all the methods in homogeneous case using warm start, except on the full MNIST and CIFAR10 data, where the centers of distributed methods are chosen randomly, with CGC using $K$-means++ initialization. For the heterogeneous data case, we use random initialization for all distributed methods and CGC. The SKL-KM method uses the more powerful $K$-means++ initialization in all the experiments. For homogeneous data all the methods run for $T = 4.000$ iterations, while for heterogeneous data the methods run for $T = 1.000$ iterations for Iris and $T = 4.000$ iterations for MNIST7 and CIFAR3 data. The results for homogeneous data are presented in Table \ref{tab:hom}, while the results for heterogeneous data are presented in Table \ref{tab:het}. We can see that our methods consistently perform on par with, or better than the other methods, highlighting the strong performance of our proposed framework across a myriad of real datasets.

\begin{table*}
\caption{ARI score on heterogeneous data. We use $B = 1$, with $\rho = 100$ on Iris and $\rho = 1000$ on MNIST7 and CIFAR3 data, for all distributed methods.}
\label{tab:het}
\begin{adjustwidth}{-1in}{-1in}
\begin{center}
\begin{small}
\begin{sc}
\begin{tabular}{lccccccc}
\toprule
 & SKL-KM & CGC-KM & LGC-KM & DGC-KM$_\rho$ & DGC-HL$_\rho$ & DGC-LL$_\rho$ & ADMM-KM \\
\midrule
Iris & $0.67 \pm 0.12$ & $0.67 \pm 0.12$ & $0.66 \pm 0.02$ & $\mathbf{0.77 \pm 0.06}$ & $0.76 \pm 0.06$ & $0.71 \pm 0.06$ & $0.76 \pm 0.02$ \\
MNIST7 & $0.45 \pm 0.0005$ & $0.47 \pm 0.03$ & $0.06 \pm 0.01$ & $0.43 \pm 0.05$ & $0.47 \pm 0.04$ & $0.47 \pm 0.02$ & $\mathbf{0.48 \pm 0.02}$ \\
CIFAR3 & $0.111 \pm 0.0$ & $\mathbf{0.113 \pm 0.003}$ & $0.016 \pm 0.006$ & $0.111 \pm 0.003$ & $0.092 \pm 0.03$ & $0.108 \pm 0.002$ & $0.107 \pm 0.002$ \\
\bottomrule
\end{tabular}
\end{sc}
\end{small}
\end{center}
\end{adjustwidth}
\vskip -0.1in
\end{table*}

\section{Conclusion}\label{sec:conclusion}

We study clustering over distributed data, where users have access to their local dataset, with the goal of obtaining a clustering of the full data. We design a family of clustering algorithms, DGC-$\mathcal{F}_\rho$, whose main advantages compared to the standard approaches are its applicability to a wide range of clustering problems, tunable computation and communication cost, as well as the ease of implementation. Theoretical studies show that the sequence of centers generated by DGC-$\mathcal{F}_\rho$ converges to fixed points and that, as $\rho$ increases, we are guaranteed to produce a clustering of the full data. We confirm our results numerically and demonstrate strong performance across a set of scenarios. Future ideas include using group lasso regularizer, known to achieve center consensus for finite values of $\rho$, e.g., \cite{network-lasso,sun-cvx_clust,armacki-personalized}, study the statistical consistency of the algorithm, e.g., \cite{pollard-kmeans,ghosh_ifca,armacki-oneshot}, as well as cluster recovery guarantees, e.g., \cite{kumar_kmeans,kmeans_awasthi,one_shot_clust}. Finally, an important direction left for future work is a comprehensive study of an inexact distributed clustering framework, where the inexactness can come in many forms, such as noisy gradient computation \cite{bertsekas-gradient,stochastic-kmeans}, communication channel noise \cite{kar-noisy-channels,armacki-noisy}, or inexact messages exchanged by users \cite{nedic-quantized,kar-quantized}. 

\section*{Acknowledgments}

The authors would like to thank Himkant Sharma (IIT Kharagpur) for his help with some of numerical experiments.

\bibliography{bibliography}

\newpage

\appendix

\section{Introduction}

The Appendix provides additional materials and proofs omitted from the main body of the paper. Appendix \ref{app:intuition} provides some intuition behind our approach. Appendix \ref{app:proofs} contains proofs omitted from the main body. Appendix \ref{app:asmpt3} shows the generality of Assumption \ref{asmpt:coerc}. Appendix~\ref{app:exp} provides additional numerical results.

\section{Algorithm Intuition}\label{app:intuition}

In this Appendix we provide some intuition behind the use of gradient-based clustering and estimating the value of $\rho_0$ from Theorem \ref{thm:cons-conv}

\paragraph{Gradient-based clustering.} As discussed in Remark \ref{rmk:gradient-intuition}, it is not always possible to design closed-form center updates for clustering methods. For ease of exposition, assume we have a single user, unit weights and the clusters are updated. Then, one ideally wants to perform the following center updates
\begin{equation}\label{eq:exact}
    x^{t+1}(k) = \argmin_{x \in \R^d}\sum_{r \in C^{t+1}(k)}f(x,y_{r}),
\end{equation} which equivalently requires solving the following system $\sum_{r \in C^{t+1}(k)}\nabla_x f(x^{t+1}(k),y_{r}) = 0$. If the loss is $K$-means, the resulting system is linear and the solution is given by $x^{t+1}(k) = |C^{t+1}(k)|^{-1}\big(\sum_{r \in C^{t+1}(k)}y_{r}\big)$, which is also \emph{optimal} for any Bregman loss, see, e.g., \cite{JMLR:v6:banerjee05b}. However, if the loss is not $K$-means or Bregman, the resulting system of equations can be highly nonlinear and a closed-form solution might not exist. For instance, consider the fair loss from Example 4. In this case $\nabla_x f(x,y) = 4\eta[1 - \frac{\eta}{\eta + \|x-y\|^2}](x-y) $, resulting in a system of equations given by
\begin{equation*}
    \sum_{r \in C^{t+1}(k)}4\eta\Big[1 - \frac{\eta}{\eta + \|x^{t+1}(k)-y_{r}\|^2}\Big](x^{t+1}(k)-y_r) = 0,
\end{equation*} which is highly nonlinear in the variable of interest, namely $x^{t+1}(k)$. It is clear that the above equation has no closed-form solution and would require an iterative solver to obtain the next centers. Instead, the gradient-based approach allows us to take $B \geq 1$ steps in the direction of the negative gradient, which is easy to compute. We can control the gap between the ideal update \eqref{eq:exact} and the gradient-based one via the parameter $B$, which allows us to take multiple gradient steps and get closer to \eqref{eq:exact}, at the cost of more computation. Similarly, it can be shown that no closed-form update satisfying \eqref{eq:exact} exists for both Huber and fair losses. Therefore, using a gradient-based approach allows for a simple and computationally cheap update rule for any differentiable non-Bregman loss. 

\paragraph{On the value of $\rho_0$.} As discussed in the main body, in general, it is not possible to provide an exact value of $\rho_0$. However, am estimate can be constructed as follows. First, in the proof of Theorem \ref{thm:cons-conv} ahead, we show $\|\bL\bx_\rho\| \leq \nicefrac{2\beta R_0}{\rho}$, where $\beta$ is the smoothness parameter, $R_0 = \max_{x \in \overline{co}(\D,\bx^0)}\|x\|$, with $\overline{co}(\D,\bx^0) \subset \R^d$ being the closure of the convex hull of the union of the joint dataset and center initialization. From Lemma \ref{lm:clust_convg}, we know that, for the consensus fixed point $\overline{\bx}$, there exists a $\epsilon_* = \epsilon_*(\overline{\bx}) > 0$, such that the clusters optimal with respect to any centers $\bx^\prime \in \R^{Kmd}$ which are $\epsilon_*$-close to $\overline{\bx} \in \R^{Kmd}$, are also optimal with respect to $\overline{\bx}$. Setting $\rho_0 = \nicefrac{\epsilon_*}{2\beta R_0}$, it follows that $\|\bL\bx_{\rho}\| \leq \epsilon_*$, implying $\bx_\rho$ is $\epsilon_*$-close to $\overline{\bx}$, for all $\rho \geq \rho_0$, guaranteeing a clustering of the full, joint data is produced for all $\rho \geq \rho_0$. The smoothness constant $\beta$ is a property of the loss $f$, independent of the data and can be estimated locally by each user. To estimate $R_0$, we first note that it suffices to find an upper bound on $\rho_0$, and proceed as follows. Each user $i \in [m]$ computes $R_i = \max_{x \in \overline{co}(\D_i,\bx_i^0)}\|x\|$, after which a distributed min-consensus algorithm \cite{jakubowicz-max,jadbabaie-min} is employed, to obtain $\overline{R} = \min_{i \in [m]}R_i$, at each user. Since $\overline{co}(\D_i,\bx_i^0) \subseteq \overline{co}(\D,\bx^0)$, it follows that $\overline{R} \leq R_0$. Setting $\overline{\rho} = \nicefrac{\epsilon_*}{2\beta\overline{R}}$, we get the desired upper bound $\overline{\rho} \geq \rho_0$. Finally, while the value of $\epsilon_*$ is impossible to estimate directly, as, per Lemma 5, it depends on the specific consensus fixed point $\overline{\bx}$ that we converge to,\footnote{Itself depending on center initialization and the dataset being clustered.} using $\overline{\rho}(\epsilon) = \nicefrac{\epsilon}{2\beta\overline{R}}$, guarantees $\|L\bx_\rho\| \leq \epsilon$, for any $\rho \geq \overline{\rho}(\epsilon)$ and any $\epsilon > 0$. Therefore, in practice, taking $\epsilon$ very small (e.g., $\epsilon \approx 10^{-6}$) and deploying our algorithm with $\rho \geq \overline{\rho}(\epsilon)$, guarantees that consensus is achieved up to numerical precision and that a clustering of the joint data will almost certainly be produced. Note that it is often impossible to explicitly quantify a parameter even in centralized clustering, where, e.g., it is well-known that Lloyd's algorithm is guaranteed to converge in finite time \cite{JMLR:v6:banerjee05b}, yet the convergence time heavily depends on the dataset and cluster initialization, e.g., \cite{Milligan1980}, and can not be quantified analytically. As such, the inability to provide an explicitly value of $\rho_0$ for which we are guaranteed to produce a clustering of the joint data is a fundamental issue inherent to the problem of clustering, rather than specific to our method.

\section{Missing proofs}\label{app:proofs}

In this section we provide the proofs omitted from the main body. Subsection~\ref{subsec:proofs1} provides proofs from Section~\ref{subsec:convergence}, 
Subsection~\ref{subsec:proofs2} provides proofs from Section~\ref{subsec:fixed-pt}, while Subsection~\ref{subsec:proofs3} provides proofs omitted from Section~\ref{subsec:consensus}.

\subsection{Proofs from Section~\ref{subsec:convergence}}\label{subsec:proofs1}

As stated in the main body, in order to prove Theorem \ref{thm:convergence}, a series of lemmas are introduced. The next result characterizes the behaviour of the distributed cost $J_\rho$.

\begin{lemma}\label{lm:J_rho-co-coerc}
    For each fixed clustering $C \in \mathcal{C}_{m,K,\D}$, the function $J_\rho$ is convex and $\beta_{L,\rho}$-smooth, with $\beta_{L,\rho} = \nicefrac{\beta}{\rho} + \lambda_{\max}(L)$.
\end{lemma}

\begin{proof}
    Recall the clustering cost $J_\rho$ from~\eqref{eq:general-decentr}. Since at least one cluster is non-empty at each user, it readily follows that $J_\rho$ is convex, as a sum of convex functions. Next, we know from~\eqref{eq:distributed-gen} that we can represent the gradient of $J_\rho$ compactly as $\nabla J_\rho(\bx,C) = \nicefrac{1}{\rho}\nabla J(\bx,C) + \bL \bx$. Therefore, for each $\bx,\bz \in \R^{Kmd}$, we have
    \begin{equation}\label{eq:split}
    \begin{aligned}
        \| \nabla J_\rho(\bx,C) - \nabla J_\rho(\bz,C)\| &\leq \nicefrac{1}{\rho}\|\nabla J(\bx,C) - \nabla J(\bz,C)\| + \|\bL(\bx - \bz)\| \\ &\leq \nicefrac{1}{\rho}\|\nabla J(\bx,C) - \nabla J(\bz,C)\| + \lambda_{\max}(\bL)\|\bx - \bz\|,
    \end{aligned}
    \end{equation} where $\lambda_{\max}(\bL)$ is the largest eigenvalue of $\bL$. We now look at $\|\nabla J(\bx,C) - \nabla J(\bz,C)\|$. First note that, for each $i \in [m]$ and $k \in [K]$, the $i,k$-th component of $\nabla J(\bx,C) \in \R^{Kmd}$ is given by $[\nabla J(\bx,C)]_{i,k} = \sum_{r \in C_i(k)}w_{i,r}\nabla f(x_i(k),y_{i,r})$. It then follows that
    \begin{align*}
        \|\nabla J(\bx,C) - \nabla J(\bz,C)\|^2 &= \sum_{i \in [m]}\sum_{k \in [K]}\| [\nabla J(\bx,C) - \nabla J(\bz,C)]_{i,k} \|^2 \\ &= \sum_{i \in [m]}\sum_{k \in [K]}\Big\|\sum_{r \in C_i(k)}w_{i,r}\big[\nabla f(x_i(k),y_{i,r}) - \nabla f(z_i(k),y_{i,r}) \big] \Big\|^2 \\ &\stackrel{(a)}{\leq} \sum_{i \in [m]}\sum_{k \in [K]}\sum_{r \in C_i(k)}\frac{w_{i,r}}{\widetilde{w}_{i,k}}\Big\|\widetilde{w}_{i,k}\big[\nabla f(x_i(k),y_{i,r}) - \nabla f(z_i(k),y_{i,r}) \big] \Big\|^2 \\ &\stackrel{(b)}{\leq} \beta^2\sum_{i \in [m]}\sum_{k \in [K]}\widetilde{w}_{i,k}^2\|x_i(k) - z_i(k)\|^2 \stackrel{(c)}{\leq} \beta^2\|\bx - \bz\|^2,
    \end{align*} where $\widetilde{w}_{i,k} = \sum_{r \in C_i(k)}w_{i,r}$, $(a)$ follows from Jensen's inequality, $(b)$ follows from $\beta$-smoothness of $f$, while $(c)$ follows from the fact that $\widetilde{w}_{i,k} < 1$. Using the properties of the Kronecker product, it can be shown that $\lambda_{\max}(\bL) = \lambda_{\max}(L)$, i.e., the largest eigenvalue of $\bL \in \R^{Kmd \times Kmd}$ corresponds to the largest eigenvalue of $L \in \R^{m \times m}$. Plugging everything back in~\eqref{eq:split} yields  
    \begin{equation*}
        \| \nabla J_\rho(\bx,C) - \nabla J_\rho(\bz,C)\| \leq \nicefrac{\beta}{\rho}\|\bx - \bz\| + \lambda_{\max}(L)\|\bx - \bz\| = \beta_{L,\rho}\|\bx - \bz\|.
    \end{equation*} It can be shown that convexity and $\beta_{L,\rho}$-Lipschitz continuous gradients together imply $\beta_{L,\rho}$-smoothness, see, e.g., \cite{lectures_on_cvxopt}. This completes the proof. 
\end{proof}

Prior to stating the next result, note that the center update~\eqref{eq:grad_local} can be represented compactly as
\begin{equation}\label{eq:grad_global}
    \bx^{t,b+1} = \bx^{t,b} - \alpha \nabla J_\rho(\bx^{t,b},C^{t+1}) = \bx^{t,b} - \alpha \left(\frac{1}{\rho}\nabla J(\bx^{t,b},C^{t+1}) + \bL\bx^{t,b}\right),
\end{equation} where $\nabla J(\bx^{t,b}, C^{t+1}) \in \R^{Kmd}$ is the vector stacking of the gradients of $H$ with respect to $\bx_i$, whose $i$-th block, for any $i \in [m]$, is given by
\begin{equation}\label{eq:grad_J}
    \left[\nabla J(\bx^{t,b},C^{t+1})\right]_i =  \nabla H(\bx_i^{t,b},C_i^{t+1}) \in \R^{Kd}. 
\end{equation} We next prove that DGC-$\mathcal{F}_\rho$ generates a non-increasing sequence of values of $J_\rho$.

\begin{lemma}\label{lm:decr}
For the sequence $\{(\bx^t,C^t)\}_{t \in \N}$, generated by Algorithm~\ref{alg:dist-grad-cl}, with $\alpha < \frac{1}{\beta_{L,\rho}}$, the resulting sequence of costs $\{J_\rho(\bx^t,C^t)\}_{t \in \N}$ is non-increasing.
\end{lemma}

\begin{proof}
First, note that \eqref{eq:reassign} together with Assumption \ref{asmpt:g&f} implies that the clustering reassignment step decreases the cost since, at every time $t \geq 1$ and user $i \in [m]$,
\begin{align*}
        H(\bx^t_i,C_i^{t+1}) = \sum_{k \in [K]}\sum_{y \in C_i^{t+1}(k)}w_yf\left(x^t_i(k),y\right) \leq \sum_{k \in [K]}\sum_{y \in C_i^t(k)}w_yf\left(x^t_i(k),y\right) = H(\bx_i^t,C_i^t),
\end{align*} while the consensus part, $\langle\bx^t,\bL\bx^t\rangle$, remains unchanged. This readily implies that
\begin{align}\label{eq:leq1}
    J_\rho(\bx^t,C^{t+1}) = \frac{1}{\rho}J(\bx^t,C^{t+1}) + \frac{1}{2}\langle \bx^t,\bL\bx^t\rangle \leq \frac{1}{\rho}J(\bx^t,C^t) + \frac{1}{2}\langle \bx^t,\bL\bx^t\rangle = J_\rho(\bx^t,C^t).
\end{align} Next, from Lemma~\ref{lm:J_rho-co-coerc} we have that, for any $t \geq 0$\footnote{Note that, starting from the center initialization $\bx^0$, we first update the clusters to obtain $C^1$. From there, we perform $E$ center updates to obtain $\bx^1$, hence for the center update step, the counter starts from $t = 0$. This is different for the cluster update step, as we do not have a clustering at time $t = 0$, so cost decrease with respect to cluster update only starts from iteration $t = 1$.} and $b = 0,\ldots,B-1$,
\begin{align*}
    J_\rho(\bx^{t,b+1},C^{t+1}) \leq J_\rho(\bx^{t,b},C^{t+1}) + \left\langle \nabla J_\rho(\bx^{t,b},C^{t+1}), \bx^{t,b+1} - \bx^{t,b} \right\rangle + \frac{\beta_{L,\rho}}{2}\|\bx^{t,b+1} - \bx^{t,b} \|^2 .
\end{align*} Using \eqref{eq:grad_global}, we get
\begin{align}
    \begin{aligned}\label{eq:recursion}
        J_\rho(\bx^{t,b+1},C^{t+1}) \leq J_\rho(\bx^{t,b},C^{t+1}) - c(\alpha) \| \nabla J_\rho(\bx^{t,b},C^{t+1})\|^2,
    \end{aligned} 
\end{align} where $c(\alpha) = \alpha \left(1 - \frac{\alpha \beta_{L,\rho}}{2} \right)$. Applying~\eqref{eq:recursion} recursively, and recalling that $\bx^{t+1} = \bx^{t,B}$, $\bx^t = \bx^{t,0}$, we get that, for any time $t \geq 0$
\begin{equation}\label{eq:contraction}
    J_\rho(\bx^{t+1},C^{t+1}) \leq J_\rho(\bx^t,C^{t+1}) - c(\alpha)\sum_{b = 0}^{B-1}\|\nabla J_\rho(\bx^{t,b},C^{t+1})\|^2. 
\end{equation} Choosing $\alpha < \frac{1}{\beta_{L,\rho}}$ guarantees that $c(\alpha) > 0$, which readily implies
\begin{equation}\label{eq:leq2}
    J_\rho(\bx^{t+1},C^{t+1}) \leq J_\rho(\bx^t,C^{t+1}).
\end{equation} Finally, combining \eqref{eq:leq1} and \eqref{eq:leq2}, we get that, for any $t \geq 1$,
\begin{align*}
        J_\rho(\bx^{t+1},C^{t+1}) \leq J_\rho(\bx^t,C^{t+1}) \leq J_\rho(\bx^t,C^t),
\end{align*} which completes the proof.
\end{proof}

\begin{remark}
    We can see the benefit of performing $B$ rounds of center update in equation~\eqref{eq:contraction}, where a higher value of $B$ leads to a stronger decrease in the cost function $J_\rho$.
\end{remark}

The next result states that if two cluster centers are sufficiently close, their set of optimal clusterings match. The proof can be found in \cite{kar2019clustering,pmlr-v162-armacki22a}.

\begin{lemma}\label{lm:clust_convg}
For every $\bx \in \R^{Kmd}$, there exists an $ \epsilon_* = \epsilon_*(\bx) > 0$, such that, for any $\bx^{\prime} \in \mathbb{R}^{Kmd}$ satisfying $\max_{i \in [m], \: k \in [K]} g(x_i(k),x_i^{\prime}(k)) < \epsilon_*$, we have $U_{\bx^{\prime}} \subset U_{\bx}$.
\end{lemma}

The next result shows that any convergent subsequence of $\{\bx^t\}_{t \in \N}$ converges to a fixed point.

\begin{lemma}\label{lm:fix_pt}
Let $\{\bx^t\}_{t \in \N}$ be a sequence generated by DGC-$\mathcal{F}_\rho$, with $\alpha < \frac{1}{\beta_{L,\rho}}$. Then, any convergent subsequence converges to a fixed point.
\end{lemma}

\begin{proof}
Let $\{\bx^{t_s} \}_{s \in \N}$ be a convergent subsequence of $\{\bx^t \}_{t \in \N}$. Let $\bx^\star \in \R^{Kmd}$ be its limit point and assume the contrary, that $\bx^\star$ is not a fixed point. By Definition~\ref{def:fix-pt}, this implies $\|\nabla J_\rho(\bx^\star,C)\| > 0$, for all $C \in U_{\bx^\star}$ As the number of possible clusterings is finite, we define
\begin{equation}\label{eq:nonz-grad}
    \epsilon_1 = \min_{C \in U_{\bx^\star}}\| \nabla J_\rho(\bx^\star,C) \| > 0.
\end{equation} Next, from the continuity of $g$, we know that, for any $\epsilon >0$, there exists a $\delta > 0$, such that, for any $x,x^\prime \in \R^d$, for which $\|x - x^\prime\| \leq \delta$, we have $g(x,x^\prime) < \epsilon$. Choose $\epsilon_* > 0$ from Lemma~\ref{lm:clust_convg}. From $\lim_{s \rightarrow \infty}\bx^{t_s} =\bx^\star$, we know that, for any fixed $\delta > 0$, there exists a sufficiently large $s_0 \in \N$, such that, for all $i \in [m]$, $k \in [K]$ and all $s \geq s_0$, $\|x^{t_s}_i(k) - x^\star_i(k)\| < \delta$. From the previous discussion, it then readily follows that there exists a $\delta_* > 0$ and a sufficiently large $s_0 \in \N$, such that $g(x^{t_s}_i(k),x^\star_i(k)) < \epsilon_*$, for all $i \in [m]$, $k \in [K]$ and $s \geq s_0$. Per Lemma~\ref{lm:clust_convg}, we then have $C_{\bx^{t_s+1}} \in U_{\bx^{t_s}} \subset U_{\bx^\star}$, for all $ s \geq s_0$. From \eqref{eq:nonz-grad}, it follows that, for all $s \geq s_0$,
\begin{equation}\label{eq:contra}
    \| \nabla J_\rho(\bx^
    \star,C^{t_s+1})\| \geq \epsilon_1. 
\end{equation} Next, using~\eqref{eq:contraction}, we have 
\begin{align*}
    \begin{aligned}
        J_\rho(\bx^{t+1},C^{t+1}) &\leq J_\rho(\bx^t,C^t) - c(\alpha)\sum_{b = 0}^{B-1}\|\nabla J_\rho(\bx^{t,b},C^{t+1}) \|^2 \\ &\leq \ldots \leq J_\rho(\bx^0,C^1) - c(\alpha)\sum_{s = 0}^t\sum_{b = 0}^{B-1}\|\nabla J_\rho(\bx^{s,e},C^{s+1}) \|^2.
    \end{aligned}
\end{align*} Rearranging, we get
\begin{align}\label{eq:sum}
\begin{aligned}
    c(\alpha)\sum_{s = 0}^t\sum_{b = 0}^{B-1}\|\nabla J_\rho(\bx^{s,e},C_{s+1}) \|^2 &\leq J_\rho(\bx^0,C^1).
\end{aligned}
\end{align} Additionally, note that
\begin{equation}\label{eq:sum_subseq}
    \sum_{j=0}^{s(t)}\|\nabla J_\rho(\bx^{t_j},C^{t_j+1}) \|^2\leq \sum_{j=0}^{s(t)}\sum_{b = 0}^{B-1}\|\nabla J_\rho(\bx^{t_j,e},C^{t_j+1}) \|^2 \leq \sum_{j = 0}^t\sum_{b = 0}^{B-1}\|\nabla J_\rho(\bx^j,C^{j+1}) \|^2, 
\end{equation} where $s(t) = \sup\{j: t_j \leq t \}$. Combining \eqref{eq:sum} and \eqref{eq:sum_subseq}, we get
\begin{equation}\label{eq:sum_subseq2}
    c(\alpha)\sum_{j=0}^{s(t)}\|\nabla J_\rho(\bx^{t_j},C^{t_j+1}) \|^2 \leq J_\rho(\bx^0,C^1).
\end{equation} Since the term on the right hand side of \eqref{eq:sum_subseq2} is finite and independent of $t$, we can take the limit as $t$ goes to infinity, to obtain
\begin{align}\label{eq:bound}
    c(\alpha)\sum_{j = 0}^\infty \|\nabla J_\rho(\bx^{t_j},C^{t_j+1}) \|^2 < \infty, 
\end{align} where we use the fact that $\lim_{t \rightarrow \infty}s(t) = \infty$. Using \eqref{eq:bound}, we get $\lim_{s \rightarrow \infty} \|\nabla J_\rho(\bx^{t_s},C^{t_s+1}) \|^2 = 0$. Next, fix an $\epsilon > 0$. By the definition of limits, there exists a $s_1 \in \N$, such that, for all $s \geq s_1$, $\|\nabla J_\rho(\bx^{t_s},C^{t_s+1}) \| < \epsilon$. On the other hand, from $\lim_{s \rightarrow\infty}\bx^{t_s} = \bx^\star$, we know that there exists a $s_2 \in \N$, such that $\|\bx^{t_s} - \bx^\star \| < \epsilon,$ for all $s \geq s_2$. As $C_{\bx^{t_s+1}} \in U_{\bx^{t_s}} \subset U_{\bx^\star}$, for all $ s \geq s_0$, we then have, for any $s \geq \max\{s_0,s_1,s_2\}$,
\begin{align*}
    \|\nabla J_\rho(\bx^\star,C^{t_s+1}) \| &\leq \|\nabla J_\rho(\bx^\star,C^{t_s+1}) - \nabla J_\rho(\bx^{t_s},C^{t_s+1}) \| + \| \nabla J_\rho(\bx^{t_s},C^{t_s+1}) \| \\ &\leq \beta_{L,\rho} \| \bx^\star - \bx^{t_s} \| + \epsilon < (\beta_{L,\rho} + 1)\epsilon,
\end{align*} where we used the Lipschitz continuity of gradients of $J_\rho$ in the second inequality. As $\epsilon > 0$ was arbitrarily chosen, we can conclude that
\begin{equation}
\label{eq:grad-limit}
    \lim_{s \rightarrow \infty}\| \nabla J_\rho(\bx^\star,C^{t_s+1}) \| = 0,
\end{equation} which clearly contradicts \eqref{eq:contra}. Hence, it follows that $\bx^\star$ is a fixed point, i.e., there exists a clustering $C \in U_{\bx^\star}$, such that $\nabla J_\rho(\bx^\star,C) = 0$.
\end{proof}

The next lemma shows that the clusters converge in finite time. The proof follows the same reasoning as the one in \cite{pmlr-v162-armacki22a}, and is omitted for brevity.

\begin{lemma}\label{lm:Clust_fin_conv}
For any convergent subsequence $\{\bx^{t_s}\}_{s \in \N}$ of $\{\bx^t\}_{t \in \N}$, generated by DGC-$\mathcal{F}_\rho$, there exists a $s_0 \in \N$, such that, for all $s \geq s_0$, $C^{t_s+1} \in \overline{U}_{\bx^\star}$, where $\bx^\star = \lim_{s \rightarrow \infty}\bx^{t_s}$.
\end{lemma}

The following lemma shows that the generated sequence of cluster centers stays bounded. 

\begin{lemma}\label{lm:bdd-seq}
The sequence of centers $\{\bx^t \}_{t\in\N}$ generated by DGC-$\mathcal{F}_\rho$ stays bounded, i.e., there exists a $M_\rho > 0$, such that $\|\bx^t\| \leq M_\rho$, for all $t \in \N$.
\end{lemma}

\begin{proof}
By Lemma \ref{lm:decr}, we know that
\begin{equation}\label{eq:contr-setup}
    \ldots \leq J_\rho(\bx^t,C^{t+1}) \leq \ldots \leq J_\rho(\bx^1,C^1) \leq J_\rho(\bx^0,C^1) < \infty.
\end{equation} Next, suppose the contrary, that the sequence of centers $\{\bx^t\}_{t \in \N}$ is unbounded. This implies the existence of a user $i \in [m]$, cluster $k \in [K]$, and a subsequence $t_s$, $s  \in \N$, such that $\lim_{s \rightarrow \infty}\|x_i^{t_s}(k)\| = \infty$. From Assumption~\ref{asmpt:graph}, if consensus for the $k$-th cluster is not reached, i.e., if there exists a $j \in [m]$ such that $\lim_{s \rightarrow\infty}\|x_i^{t_s}(k) - x_j^{t_s}(k)\| > 0$, then clearly $\lim_{s \rightarrow \infty}\langle \bx^{t_s},\bL\bx^{t_s} \rangle \rightarrow \infty$, which implies $J_\rho(\bx^{t_s},C^{t_s}) \rightarrow \infty$, contradicting~\eqref{eq:contr-setup}. 

Therefore, it must be that $\lim_{s \rightarrow \infty}\|x_i^{t_s}(k) - x_j^{t_s}(k)\| = 0$, for all $j \in [m]$. For each $s \in \N$, let $\overline{t}_s = \max\big\{t \leq t_s: \text{ there exists a } j \in [m] \text{ such that } C_j^t(k) \ne \emptyset \big\}$, i.e., $\overline{t}_s$ is the largest element in the sequence prior to $t_s$, such that the $k$-th cluster of at least one user is non-empty.  We then have the following possibilities:
\begin{enumerate}
    \item If $\overline{t}_s = t_s$, then $x_j^{t_s}(k) = x_j^{\overline{t}_s}(k)$, for all $j \in [m]$.
    \item If $\overline{t}_s < t_s$, then, recalling the update rule \eqref{eq:grad_local}, for all $j \in [m]$, we have 
    \begin{align*}
        x_j^{t_s}(k) &= x_j^{t_s-1,B-1}(k) - \alpha\sum_{l \in \calN_j}\left(x_j^{t_s - 1,B-1}(k) - x_l^{t_s - 1,B-1}(k)\right) \\ &= x_j^{t_s - 1,B-2}(k) - \alpha\sum_{l \in \calN_j}\left(x_j^{t_s - 1,B-1}(k) - x_l^{t_s - 1,B-1}(k) + x_j^{t_s - 1,B-2}(k) - x_l^{t_s - 1,B-2}(k)\right) \\ &= \ldots = x_j^{t_s - 1}(k) - \alpha\sum_{b = 0}^{B - 1}\sum_{l \in \calN_j}\left(x_j^{t_s - 1,b}(k) - x_l^{t_s - 1,b}(k)\right) \\ &= \ldots = x_j^{\overline{t}_s}(k) - \alpha\sum_{r = 1}^{t_s - \overline{t}_s}\sum_{b = 0}^{B - 1}\sum_{l \in \calN_j}\left(x_j^{t_s - r,b}(k) - x_l^{t_s - r,b}(k)\right).
    \end{align*} By the definition of $\overline{t}_s$, if $\overline{t}_s < t_s$, it follows that, for the $k$-th center, the algorithm only performs $B(t_s - \overline{t}_s)$ consensus steps between times $\overline{t}_s$ and $t_s$, i.e., only diffuses the $k$-th center estimates $x^{\overline{t}_s}_j(k)$, $j \in [m]$, across the network\footnote{To be more precise, since the algorithm performs $E(t_s - \overline{t}_s)$ consensus steps, the diffusion is performed across the $E(t_s - \overline{t}_s)$-hop neighbours. For a vertex $i$, the set of $K$-hop neighbours of $i$ is the set of vertices that can be reached from $i$ by traversing at most $K$ edges, see, e.g., \cite{chung1997spectral,cvetkovic_rowlinson_simic_1997}.}.
\end{enumerate} From $1.$ and $2.$ we can readily conclude that $x_j^{t_s}(k)$ only depends on the $k$-th center estimates at time $\overline{t}_s$, $x_l^{\overline{t}_s}(k)$, $l \in [m]$, for all users.  From the preceding discussion, the facts that $\lim_{s \rightarrow \infty}\|x_i^{t_s}(k)\|=\infty$ and $\lim_{s \rightarrow \infty}\|x_i^{t_s}(k) - x_j^{t_s}(k)\| = 0$, for all $j \in [m]$, we can readily conclude that $\|x_j^{\overline{t}_s}(k) \| \rightarrow \infty$, for all $j \in [m]$. For a center $\bx\in \mathbb R^{Kmd}$ and clustering $C$, define
\begin{equation*}
    H_k(\bx,C) = \sum_{j \in [m]}\sum_{r \in C_j(k)}w_{j,r} f(x_j(k),y_{j,r}),
\end{equation*} i.e., for a given center and clustering, $H_k$ defines the cost associated with the $k$-th cluster across all the users. Combining the facts that $\|x_j^{\overline{t}_s}(k) \| \rightarrow \infty$, for all $j \in [m]$, and that for every $s \in \N$, there exists a $j \in [m]$, such that $C_j^{\overline{t}_s}(k) \neq \emptyset$, with Assumption~\ref{asmpt:coerc}, we get
\begin{equation*}
    \lim_{s \rightarrow \infty} H_k\left(\bx^{\overline{t}_s},C^{\overline{t}_s}\right) = \lim_{s \rightarrow \infty} \sum_{j \in [m]}\sum_{r \in C_j^{\overline{t}_s}(k)}w_{j,r}f(x_j^{\overline{t}_s}(k),y_{j,r}) = \infty.
\end{equation*} It is easy to see that unboundness of $H_k$ implies unboundedness of $J_\rho$, i.e., we have \linebreak $\lim_{s \rightarrow \infty} J_\rho(\bx^{\overline{t}_s}, C^{\overline{t}_s})=\infty$, clearly contradicting~\eqref{eq:contr-setup}. Therefore, the desired claim follows.
\end{proof}

The next lemma shows that, if a point in the sequence of centers is sufficiently close to a fixed point, then all the subsequent points remain in the neighborhood of the fixed point.

\begin{lemma}\label{lm:final-step}
Let $\{\bx^t\}_{t \in \N}$ be the sequence of centers generated by DGC-$\mathcal{F}_\rho$, with the step-size satisfying $\alpha < \frac{1}{\beta_{L,\rho}}$. Let $\bx^\star \in \R^{Kmd}$ be a fixed point, in the sense of Definition~\ref{def:fix-pt}. Then, there exists an $\epsilon^\star > 0$, for which, for all $\epsilon \in (0,\epsilon^\star)$, there exists a $ t_{\epsilon} \in \N$, such that, if $\|\bx^{t_0} - \bx^\star \| \leq \epsilon$, for some $t_0 > t_{\epsilon}$, then $\|\bx^t - \bx^\star \| \leq \epsilon$, for all $t \geq t_0$.
\end{lemma} 

\begin{proof}
Recall that, by Lemma \ref{lm:decr}, the sequence of costs $\{J_\rho(\bx^t,C^t) \}_{t \in \N}$ is non-increasing. Moreover, since $J_\rho \geq 0$, we know that the limit of the sequence of costs exists and is finite. Let
\begin{equation}\label{eq:J*def}
    J_\rho^\star = \lim_{t \rightarrow \infty}J_\rho(\bx^t,C^t).
\end{equation} By assumption, $\overline{U}_{\bx^\star} \ne \emptyset$. From the definition of $\overline{U}_{\bx^\star}$, for all $C \in U_{\bx^\star} \setminus \overline{U}_{\bx^\star}$ we have
\begin{equation}\label{eq:grad-nonz}
    \|\nabla J_\rho(\bx^\star,C) \| > 0.
\end{equation} As $U_{\bx^\star}$ is a finite set, we can define $\epsilon_1 = \min_{C \in U_{\bx^\star} \setminus \overline{U}_{\bx^\star}} \|\nabla J_\rho(\bx^\star,C)\| > 0$. Let $\epsilon_* > 0$ be such that Lemma \ref{lm:clust_convg} holds. From the continuity of $g$, we know that there exists a $\delta_* > 0$, such that, for all $\bx \in \R^{Kmd}$, $i \in [m]$, $k \in [K]$,
\begin{equation}\label{eq:delta}
    \|\bx - \bx^\star\| < \delta_* \implies g(x_i(k),x^\star_i(k)) < \epsilon_* .
\end{equation} Define
\begin{equation}\label{eq:eps*}
    \epsilon^\star = \min\bigg\{\delta_*, \frac{\epsilon_1}{\beta_{L,\rho}} \bigg\}. 
\end{equation} For an arbitrary $\epsilon \in (0,\epsilon^\star)$, let $t_0 \in \N$ be such that, for all $t \geq t_0$, 
\begin{equation}\label{eq:condition}
    J_\rho(\bx^t,C^t) \leq J^\star_\rho + \frac{c(\alpha)}{2}(\epsilon_1 - \beta_{L,\rho}\epsilon)^2, 
\end{equation} with $c(\alpha)$ defined as in Lemma \ref{lm:decr}. Note that such a choice of $t_0$ is possible, from \eqref{eq:J*def} and the fact that $(\epsilon_1 - \beta_{L,\rho}\epsilon)^2 > 0$. Our goal now is to show that, for a fixed $\epsilon\in (0,\epsilon^\star)$, if for some $t$ such that $t \geq t_0$ and $\|\bx^t - \bx^\star\| < \epsilon$, then $\|\bx^{t+1} - \bx^\star \| < \epsilon$.

First note that, if $t \geq t_0$ and $\|\bx^t - \bx^\star\| < \epsilon$, it holds that $C^{t+1} \in \overline{U}_{\bx^\star}$. To see this, assume the contrary, $\| \bx^t - \bx^\star \| < \epsilon$ and $C^{t+1} \notin \overline{U}_{\bx^\star}$. It follows from \eqref{eq:eps*} that $\| \bx^t - \bx^\star \| < \delta_*$. From \eqref{eq:delta} and Lemma \ref{lm:clust_convg}, we then have $U_{\bx^t} \subset U_{\bx^\star}$, and hence, $C^{t+1} \in U_{\bx^\star}$. Using Lipschitz continuity of gradients of $J_\rho$, we get
\begin{equation}\label{eq:step1}
    \|\nabla J_\rho(\bx^t,C^{t+1}) - \nabla J_\rho(\bx^\star,C^{t+1}) \| \leq \beta_{L,\rho}\|\bx^t - \bx^\star\| \leq \beta_{L,\rho}\epsilon.
\end{equation} As $C^{t+1} \notin \overline{U}_{\bx^\star}$, from \eqref{eq:grad-nonz}, we have
\begin{equation}\label{eq:step2}
    \|\nabla J_\rho(\bx^\star,C^{t+1}) \| \geq \epsilon_1.
\end{equation} Applying the triangle inequality, \eqref{eq:step1} and \eqref{eq:step2}, we get 
\begin{equation}\label{eq:above}
    \|\nabla J_\rho(\bx^t,C^{t+1}) \| \geq \epsilon_1 - \beta_{L,\rho}\epsilon.
\end{equation} Note that by \eqref{eq:eps*}, the right-hand side of \eqref{eq:above} is positive. Combining \eqref{eq:contraction},~\eqref{eq:condition},~\eqref{eq:above}, we get
\begin{align*}
    \begin{aligned}
        J_\rho(\bx^{t+1},C^{t+1}) &\leq J_\rho(\bx^t,C^t) - c(\alpha)\|\nabla J_\rho(\bx^t,C^{t+1}) \|^2 \\ &\leq J_\rho^\star + \frac{c(\alpha)}{2}(\epsilon_1 - \beta_{L,\rho}\epsilon)^2 - c(\alpha)\| \nabla J_\rho(\bx^t,C^{t+1}) \|^2 \\ &\leq J_\rho^\star + \frac{c(\alpha)}{2}(\epsilon_1 - \beta_{L,\rho}\epsilon)^2 - c(\alpha)(\epsilon_1 - \beta_{L,\rho}\epsilon)^2 < J_\rho^\star, 
    \end{aligned} 
\end{align*} which is a contradiction. Hence, $C^{t+1} \in \overline{U}_{\bx^\star}$. Next, it can be shown that convexity and $\beta_{L,\rho}$-smoothness of $J_\rho$ together imply $\beta_{L,\rho}$-co-coercivity of $J_\rho$, i.e., for each fixed clustering $C \in \mathcal{C}_{m,K,\D}$ and $\bx,\bz \in \R^{Kmd}$, we have
\begin{equation*}
    \langle \nabla J_\rho(\bx,C) - \nabla J_\rho(\bz,C), \bx - \bz \rangle \geq \frac{1}{\beta_{L,\rho}}\|\nabla J_\rho(\bx,C) - \nabla J_\rho(\bz,C) \|^2.
\end{equation*} For a formal account of this result, see, e.g., \cite{lectures_on_cvxopt}. Combining this fact with the update rule \eqref{eq:grad_global}, and the fact that $C^{t+1} \in \overline{U}_{\bx^\star}$, we have
\begin{align*}
    \begin{aligned}
        \|\bx^{t+1} - \bx^\star \|^2 &= \|\bx^{t,B-1} - \alpha \nabla J_\rho(\bx^{t,B-1},C^{t+1}) - \bx^\star \|^2 \\ &= \|\bx^{t,B-1} - \bx^\star \|^2 + \alpha^2\|\nabla J_\rho(\bx^{t,B-1},C^{t+1}) \|^2 - 2\alpha\langle \nabla J_\rho(\bx^{t,B-1},C^{t+1}), \bx^t - \bx^\star \rangle \\ &\leq \|\bx^{t,B-1} - \bx^\star \|^2 - \alpha\left(\frac{1}{\beta_{L,\rho}} - \alpha \right)\|\nabla J_\rho(\bx^{t,B-1},C^{t+1}) \|^2.
    \end{aligned}
\end{align*} Repeating the argument recursively and recalling that $\bx^{t,0} = \bx^t$, we get
\begin{equation}\label{eq:seq-decr}
    \|\bx^{t+1} - \bx^\star \|^2 \leq \|\bx^t - \bx^\star\|^2 - \alpha\left(\frac{1}{\beta_{L,\rho}} - \alpha\right)\sum_{b = 0}^{B-1}\|\nabla J_\rho(\bx^{t,b},C^{t+1})\| \leq \|\bx^t - \bx^\star\|^2 < \epsilon^2.
\end{equation} where second inequality follows from the step-size choice $\alpha < \frac{1}{\beta_{L,\rho}}$. Therefore, we have shown that $\| \bx^t - \bx^\star \| < \epsilon$ implies $\| \bx^{t+1} - \bx^\star \| < \epsilon$. The same result holds for all $s > t$ inductively.
\end{proof}

\begin{remark}
    We can again see the benefit of performing $B$ rounds of center update in~\eqref{eq:seq-decr}, where higher values of $B$ lead to stronger decrease in distance of centers from a fixed point.
\end{remark}

We are now ready to prove our main result.

\begin{proof}[Proof of Theorem~\ref{thm:convergence}]
By Lemma \ref{lm:decr} and the fact that the corresponding sequence of costs $\{J_\rho(\bx^t,C^t)\}_{t \in \N}$ is nonnegative, the monotone convergence theorem states that this sequence converges to some $J^\star_\rho \geq 0$. On the other hand, by Bolzano-Weierstrass theorem and Lemma \ref{lm:bdd-seq}, the sequence $\{\bx^t\}_{t \in \N}$ has a convergent subsequence, $\{\bx^{t_s}\}_{s \in \N}$, with some $\bx^\star \in \mathbb{R}^{Kmd}$ as its limit. From the continuity of $J_\rho$ and convergence of $\bx^{t_s}$, we can then conclude that $J^\star_\rho = \lim_{s \rightarrow \infty}J_\rho(\bx^{t_s},C^{t_s}) = J(\bx^\star, C^\star)$. Lemma \ref{lm:fix_pt} then implies that $\bx^\star$ is a fixed point. Finally, Lemmas \ref{lm:Clust_fin_conv} and \ref{lm:final-step} imply the convergence of the full sequence, i.e., $\lim_{t \rightarrow \infty}\bx^t = \bx^\star$. Convergence of clusters in finite  time is implied by Lemmas \ref{lm:clust_convg}, \ref{lm:Clust_fin_conv}.
\end{proof}

\subsection{Proofs from Section \ref{subsec:fixed-pt}}\label{subsec:proofs2}

\noindent \textbf{Example of Bregman distances that are squares of metrics}. For the squared Euclidean norm $f(x,y) = \|x - y \|^2$, we have $g(x,y) = \sqrt{f(x,y)} = \|x - y\|$. Similarly, for the Mahalanobis-like Bregman distance $f(x,y) = \|x - y\|_A^2$, for a positive definite matrix $A$, we have $g(x,y) = \sqrt{f(x,y)} = \|x - y\|_A$. For further examples, see \cite{pmlr-v162-armacki22a}. 

\begin{proof}[Proof of Lemma \ref{lm:fix_pt}]
    By Definition~\ref{def:fix-pt}, we know that $(\bx^\star,C^\star)$ must satisfy
    \begin{equation}\label{eq:fix-pt1}
        0 = \nabla J_\rho(\bx^\star,C^\star) = \frac{1}{\rho}J(\bx^\star,C^\star) + \bL\bx^\star.
    \end{equation} By the definition of $J(\bx^\star,C^\star)$ and $\bL$, it is not hard to see that, for each $i \in [m]$, the $k$-th component of the gradient of $J_\rho(\bx^\star,C^\star)$, for any $k \in [K]$, is given by
    \begin{equation}\label{eq:fix-pt2}
        [\nabla J(\bx^\star,C^\star)]_{i,k} = \frac{1}{\rho}\sum_{r \in C_i^\star(k)}w_{i,r}\nabla f(x_i^\star(k),y_{i,r}) + \sum_{j \in \calN_i}(x_i^\star(k) - x_j^\star(k)).
    \end{equation} Using the definition of Bregman divergence, we have
    \begin{equation}\label{eq:fix-pt3}
        \nabla_x f(x,y) = -\nabla \psi(x) + \nabla \psi(x) + \nabla^2\psi(x)(x - y) = \nabla^2\psi(x)(x - y). 
    \end{equation} Combining~\eqref{eq:fix-pt1},~\eqref{eq:fix-pt2} and~\eqref{eq:fix-pt3}, for any $i \in [m]$ and $k \in [K]$, we get
    \begin{align}
        &\frac{1}{\rho}\sum_{r \in C_i^\star(k)}w_{i,r}\nabla^2\psi(x_i^\star(k))(x_i^\star(k) - y_{i,r}) + \sum_{j \in \calN_i}(x^\star_i(k) - x^\star_j(k)) = 0 \iff \nonumber \\ &\left(\frac{1}{\rho}\nabla^2\psi(x_i^\star(k))\sum_{r \in C_i^\star(k)}w_{i,r} + |\calN_i|I_{d} \right)x^\star_i(k) = \frac{1}{\rho}\sum_{r \in C_i^\star(k)}w_{i,r}y_{i,r} + \sum_{j \in \calN_i}x^\star_j(k). \label{eq:fix-pt4}
    \end{align} According to Assumption~\ref{asmpt:psi}, the matrix $P_{i,k} = \frac{1}{\rho}\nabla^2\psi(x_i^\star(k))\sum_{r \in C_i^\star(k)}w_{i,r} + |\calN_i|I_{d}$ is positive definite, and hence invertible. Multiplying both sides of~\eqref{eq:fix-pt4} with $P_{i,k}^{-1}$ completes the proof.
\end{proof}

\subsection{Proofs from Section~\ref{subsec:consensus}}\label{subsec:proofs3}

\begin{lemma}\label{lm:asmpt-satisfied}
    Assumption \ref{asmpt:dist-norm} is satisfied for $K$-means, Huber Logistic and Fair loss functions.
\end{lemma}
\begin{proof}
    For $f(x,y) = \frac{1}{2}\|x - y\|^2$, we have $\nabla f(x,y) = x - y$, i.e., $\gamma(x,y) = 1$. For $f(x,y) = \phi_\delta(\|x - y\|)$, we have 
    \begin{equation*}
        \nabla f(x,y) = \begin{cases}
            x - y, & \|x - y\| \leq \delta \\
            \frac{\delta(x - y)}{\|x - y\|}, & \|x - y\| > \delta
        \end{cases}
    \end{equation*} i.e.,
    \begin{equation*}
        \gamma(x,y) = \begin{cases}
            1, & \|x - y\| \leq \delta \\
            \frac{\delta}{\|x - y\|}, & \|x - y\| > \delta
        \end{cases}.
    \end{equation*} For $f(x,y) = \log(1 + \exp(g(x,y)^2))$, we have $\nabla f(x,y) = \frac{2(x - y)}{1 + \exp(-\|x - y\|^2)}$, i.e., $\gamma(x,y) = \frac{2}{1 + \exp(-\|x - y\|^2)}$. Finally, for $f(x,y) = h_\eta(\|x - y\|)$, we have $\nabla f(x,y) = 4\eta[1 - \frac{\eta}{\eta + \|x-y\|^2}](x-y)$, i.e., $\gamma(x,y) = 4\eta[1 - \frac{\eta}{\eta + \|x-y\|^2}]$, which completes the proof.
\end{proof}

The next result shows that fixed points of DGC-$\mathcal{F}_\rho$ remain in $\overline{co}(\D,\bx^0)$\footnote{We use the shorthand notation $co(\D,\bx^0)$ to denote the convex hull of the data $\D$ and initial centers $x_i^0(k)$, $i \in [m]$, $k \in [K]$.}, for each $\rho$.

\begin{lemma}\label{lm:bdd-fix-pt}
    Let Assumption \ref{asmpt:dist-norm} hold. Then for any $\rho \geq 1$, fixed points $\bx_\rho$ of DGC-$\mathcal{F}_\rho$ satisfy $x_{i,\rho}(k) \in \overline{co}(\D,\bx^0)$, for all $i \in [m]$, $k\in[K]$, where $\bx^0 \in \R^{Kmd}$ is the center initialization. 
\end{lemma}

\begin{proof}
    We will show a stronger result, namely, that for each fixed $\rho$, the sequence of centers $\{\bx^t \}_{t\in \N}$ generated by DGC-$\mathcal{F}_\rho$ stays in $co(\D,\bx^0)$. We prove the claim by induction. Clearly, for each $i \in [m]$ and $k \in [K]$, we have $x_i^0(k) \in co(\D,\bx^0)$. 
    
    Next, assume that, for some $t > 0$ and each $i \in [m]$ and $k \in [K]$, we have $x_i^t(k) \in co(\D,\bx^0)$. Recalling the update equation~\eqref{eq:grad_local}, using Assumption~\ref{asmpt:dist-norm} and the fact that $x_i^{t,0}(k) = x_i^t(k)$, it follows that, for each $i \in [m]$ and $k \in [K]$    
    \begin{align}
    x_i^{t,1}(k) &= x_i^t(k) - \alpha\bigg( \sum_{j \in \calN_i} \left[x^{t}_i(k) - x^{t}_j(k) \right]+ \frac{1}{\rho}\hspace{-0.5em}\sum_{r \in C_i^{t+1}(k)}\hspace{-1em}w_{i,r} \gamma^t_{i,r}(k)(x_i^{t}(k) - y_{i,r}) \bigg) \nonumber \\ &= \bigg(1 - \alpha\big(|\mathcal{N}_i| + \nicefrac{1}{\rho}\hspace{-1em}\sum_{r \in C_i^{t+1}(k)}\hspace{-1em}w_{i,r}\gamma^t_{i,r}(k)\big)\bigg)x_i^t(k) + \alpha \sum_{j \in \calN_i} x^{t}_j(k) + \nicefrac{\alpha}{\rho}\hspace{-1em}\sum_{r \in C_i^{t+1}(k)}\hspace{-1em}w_{i,r} \gamma^t_{i,r}(k)y_{i,r} \label{eq:conv-comb}, 
    \end{align} where we use $\gamma^t_{i,r}(k)$ as a shorthand notation for $\gamma(x^t_i(k),y_{i,r})$. It can be readily seen that \eqref{eq:conv-comb} is a convex combination of $x_i^{t}(k), \: x_j^t(k)$ and $y_{i,r}$, $j \in \mathcal{N}_i$, $r \in C_i^{t+1}(k)$, for the step-size choice $\alpha < \frac{1}{\nicefrac{1}{\rho}\sum_{r \in C_i^{t+1}}w_{i,r}\gamma^t_{i,r}(k) + |\mathcal{N}_i|}$. It is easy to see that $|\mathcal{N}_i| \leq \lambda_{\max}(L)$ and, using the results from Lemma \ref{lm:asmpt-satisfied}, it can be shown that $\gamma(x,y) \leq \beta$, for all four loss functions ($K$-means, Huber, Logistic and Fair), where we recall that $\beta$ is the smoothness parameter. Finally, recalling that our step-size satisfies $\alpha < \frac{1}{\nicefrac{\beta}{\rho} + \lambda_{\max}(L)} \leq \frac{1}{\nicefrac{1}{\rho}\sum_{r \in C_i^{t+1}}w_{i,r}\gamma^t_{i,r}(k) + |\mathcal{N}_i|}$, it readily follows from \eqref{eq:conv-comb} that $x^{t,1}_i(k)$ is a convex combination of elements from $co(\D,\bx^0)$, by induction hypothesis. Using the same arguments, we can easily show that $x_i^{t,b}(k) \in co(\D,\bx^0)$, for all $i \in [m]$, $k \in [K]$ and $b = 2,\ldots,B-1$. Since $\bx^{t+1} = \bx^{t,B}$, it follows that $x^{t+1}_i(k) = x^{t,B}_i(k) \in co(\D,\bx^0)$, which completes the induction proof. 
    
    Noting that the set $\overline{co}(\D,\bx^0)$ is closed, it readily follows that the limit points of sequences generated by DGC-$\mathcal{F}_\rho$ stay in $\overline{co}(\D,\bx^0)$. According to Theorem~\ref{thm:convergence}, the limit points are the fixed points, which completes the proof.
\end{proof}

We next prove Lemma \ref{lm:bregman-cons-fp}.

\begin{proof}[Proof of Lemma \ref{lm:bregman-cons-fp}]
    Specializing \eqref{eq:zero-grad2} for $f(x,y) = \|x - y\|^2$, we get, for all $k \in [K]$
    \begin{equation*}
        2\sum_{i \in [m]}\sum_{r \in \overline{C}_i(k)} w_{i,r}(x(k)- y_{i,r}) = 0.
    \end{equation*} Defining $W_k = \sum_{i \in [m]}\sum_{r \in \overline{C}_i(k)}w_{i,r}$ and $C(k) = \cup_{i \in [m]}\overline{C}_i(k)$, the claim readily follows.
\end{proof}

We next prove Lemma \ref{lm:cons-fp-D}.

\begin{proof}[Proof of Lemma \ref{lm:cons-fp-D}]
    Assume the contrary, that there exists a cluster $k \in [K]$ such that $C(k) \neq \emptyset$ and $x(k) \notin \overline{co}(\D)$. Define $\widetilde{x}(k)$ to be the projection of $x(k)$ onto $\overline{co}(C(k))$, with respect to the Euclidean distance, i.e., $\widetilde{x}(k) = \argmin_{y \in \overline{co}(C(k))}\|x(k) - y\|$. Note that the projection is well defined, as the distance metric is induced by an inner product and $\overline{co}(C(k))$ is a non-empty, closed, convex set. By Assumption~\ref{asmpt:g&f}, for all $r \in C(k)$, we have 
    \begin{equation*}
        g(\widetilde{x}(k),y_{i,r}) \leq g(x(k),y_{i,r}) \implies f(\widetilde{x}(k),y_{i,r}) \leq f(x(k),y_{i,r}),
    \end{equation*} where we used the fact that $\widetilde{x}(k)$ is the projection of $x(k)$ onto $\overline{co}(C(k))$. 
    
    If for all $r \in C(k)$ we have $\|\widetilde{x}(k) - y_{i,r}\| = \|x(k) - y_{i,r}\|$, it readily follows that $\|\widetilde{x}(k) - \widetilde{y}\| = \|x(k) - \widetilde{y}\|$, for all $\widetilde{y} \in co(C(k))$. If $\widetilde{x}(k) \in co(C(k))$, it then follows that $\|x(k) - \widetilde{x}(k)\| = \|\widetilde{x}(k) - \widetilde{x}(k)\| = 0$, implying that $x(k) \in co(C(k))$, which can not be, as $co(C(k)) \subseteq \overline{co}(\D)$ and $x(k) \notin \overline{co}(\D)$. Similarly, if $\widetilde{x}(k) \in \partial co(C(k))$, where $\partial co(C(k))$ denotes the boundary of $co(C(k))$, by definition of boundary, there exists a sequence $\{\widetilde{y}_n\}_{n \in \N} \subset co(C(k))$ that converges to $\widetilde{x}(k)$. From the fact that $\|\widetilde{x}(k) - \widetilde{y}_n\| = \|x(k) - \widetilde{y}_n\|$, for all $n \in \N$, we have $\lim_{n \rightarrow \infty}\|x(k) - \widetilde{y}_n\| = \lim_{n \rightarrow \infty}\|\widetilde{x}(k) -  \widetilde{y}_n\| = 0$, resulting in $x(k) \in \partial co(C(k))$, which again can not be. 
    
    Therefore, there must exist a $r \in C(k)$ such that $\|\widetilde{x}(k) - y_{i,r}\| < \|x(k) - y_{i,r}\|$, implying that $f(\widetilde{x}(k),y_{i,r}) < f(x(k),y_{i,r})$. As such, we have
    \begin{equation*}
        \sum_{r \in C(k)}w_{i,r}f(\widetilde{x}(k),y_{i,r}) < \sum_{r \in C(k)}w_{i,r}f(x(k),y_{i,r}).
    \end{equation*} Defining $\widehat{\bx} = \begin{bmatrix} \widehat{x}(1)^\top & \ldots & \widehat{x}(K)^\top \end{bmatrix}^\top$, where $\widehat{x}(l) = x(l)$ for all $l \neq k$ and $\widehat{x}(k) = \widetilde{x}(k)$, it can be readily observed that 
    \begin{equation}\label{eq:contr-fixpt}
        H(\widehat{\bx},C) < H(\bx,C),   
    \end{equation} where $H$ is the centralized cost from~\eqref{eq:general-clust}. From the convexity of $f$ (Assumption~\ref{asmpt:coerc}) and Definition~\ref{def:consensus-fixed_pt}, it readily follows that a consensus fixed point must satisfy $\bx \in \argmin_{\by \in \R^{Kd}}H(\by,C),$ which is clearly violated in~\eqref{eq:contr-fixpt}. Therefore, the claim follows.
\end{proof}

Prior to proving Theorem \ref{thm:cons-conv}, consider the function $F_c(\bx) = J(\bx,C)$, where $C$ is a fixed clustering of the dataset $\D$ into $K$ clusters. From the finiteness of $\D$, we know that there is a finite number of distinct partitions and hence a finite number of distinct functions $F_c(\bx)$, while from Assumption~\ref{asmpt:coerc} it follows that all functions $F_c(x)$ are coercive. This in turns implies the existence of a global minimizer of $F_c(\bx) $, i.e., a point $\bz_c \in \R^{Kmd}$, such that $F_c(\bz_c) = \min_{\bx \in \R^{Kmd}}F_c(\bx)$. It can readily be seen that, if some of the partitions of the clustering $C$ are empty, we can set the corresponding centers to be any value without changing the minimum, e.g., if $C_i(k) = \emptyset$, the point $\bz^\prime_c \in \R^{Kmd}$, given by $z^\prime_{c,j}(l) = z_{c,j}(l)$, for all $l \in [K]$ and $j \neq i$, and 
\begin{equation*}
    z^\prime_{c,i}(l) = \begin{cases}
        z_{c,i}(l), & l \neq k \\
        0, & l = k
    \end{cases},    
\end{equation*} then $\bz^\prime_c = \argmin_{\bx \in \R^{Kmd}}F_c(\bx)$. Recall that $\mathcal{C}_{K,\D}$ denotes the set of all $K$-partitions of the full dataset $\D$. Combining this with the coercivity of $F$ and the finiteness of $\D$, it readily follows that, for each $C \in \mathcal{C}_{K,\D}$, we can find a minimizer $\bz^\prime_c = \argmin_{\bx \in \R^{Kmd}}F_c(\bx)$, such that $\max_{C \in \mathcal{C}_{K,\D}}\|\bz^\prime_c \| < \infty$. Moreover, Lemma~\ref{lm:cons-fp-D} guarantees an even stronger result, that we can find $\bz^\prime_c$ such that $\bz^\prime_c = \argmin_{\bx \in \R^{Kmd}}F_c(\bx)$ and $z^\prime(k) \in \overline{co}(\D)$, for each $k \in [K]$. Since $\overline{co}(\D) \subset \overline{co}(\D,\bx^0)$ and $\overline{co}(\D,\bx^0)$ is a compact set (follows from finiteness of $\D$), there exits a finite $R_0 > 0$, such that $R_0 = \max_{\bx \in \overline{co}(\D,\bx^0)}\|\bx\|$. 

We are now ready to prove Theorem~\ref{thm:cons-conv}.

\begin{proof}[Proof of Theorem \ref{thm:cons-conv}]
    We start by showing that consensus is achieved with rate $\mathcal{O}(\nicefrac{1}{\rho})$. Recall the clustering cost in \eqref{eq:distributed-gen}. By definition, we know that, for any $C_{\rho} \in \overline{U}_{\bx_{\rho}}$
    \begin{equation}\label{eq:consensus-grad}
        \nabla J_{\rho}(\bx_{\rho},C_{\rho}) = 0 \iff L\bx_{\rho} = -\frac{1}{\rho}\nabla J(\bx_{\rho},C_{\rho}).
    \end{equation} Consider the function $F_{\rho}(\bx) = J(\bx,C_{\rho})$. From the preceding discussion and the fact that $\overline{U}_{\bx_\rho}$ is finite, it follows that for each $\rho \geq 1$, we can obtain a global minima $\bz_\rho$ of $F_\rho$, such that $\bz_\rho \in \overline{co}(\D)$. From \eqref{eq:consensus-grad} and Assumption \ref{asmpt:dist-norm}, we get $\|L\bx_{\rho}\| = \frac{1}{{\rho}}\|\nabla F_\rho(\bx_{\rho})\| \leq \frac{\beta}{\rho}\|\bx_\rho - \bz_{\rho}\| \leq \frac{2\beta R_0}{\rho}$, where the last inequality follows from Lemma \ref{lm:bdd-fix-pt}, proving the first claim.

     Denote the consensus point by $\bx \in \R^{Kd}$, i.e., from the first part we know that $\lim_{\rho \rightarrow \infty}\bx_{i,\rho} = \bx$, for all $i \in [m]$. Define $\overline{\bx} = 1_m \otimes \bx \in \R^{Kmd}$, and note that it satisfies point $1)$ in Definition \ref{def:consensus-fixed_pt}. To prove point $2)$, it remains to show that there exists a clustering $C \in U_{\overline{\bx}}$, such that $\mathbf{1}^\top\nabla J(\overline{\bx},C) = 0$. To that end, assume the contrary, that $\|\mathbf{1}^\top\nabla J(\overline{\bx},C)\| > 0$, for any $C \in U_{\overline{\bx}}$. As the number of possible clusterings is finite, we know that there exists an $\epsilon > 0$, such that $\min_{C \in U_{\overline{\bx}}}\|\mathbf{1}^\top\nabla J(\overline{\bx},C)\| = \epsilon > 0$. Next, note that from $\overline{\bx} = \lim_{\rho \rightarrow\infty}\bx_\rho$, and Lemma \ref{lm:clust_convg}, there exists a $\rho_0 \geq 1$, such that $U_{\bx_\rho} \subseteq U_{\overline{\bx}}$, for all $\rho \geq \rho_0$. As $\bx_\rho$ is a fixed point of $J_\rho$, we know that, for some $C_\rho \in U_{\bx_\rho}$, $\rho\nabla J_{\rho}(\bx_\rho,C_\rho) = \nabla J(\bx_\rho,C_\rho) + \rho\bL\bx_\rho = 0$, which readily implies $\nabla J(\overline{\bx},C_\rho) = \nabla J(\overline{\bx},C_\rho) - \nabla J(\bx_\rho,C_\rho) - \rho\bL\bx_\rho$. Using the fact that $\mathbf{1}^\top \bL = 0$, we get $\mathbf{1}^\top \nabla J(\overline{\bx},C_\rho) = \mathbf{1}^\top \left( \nabla J(\overline{\bx},C_\rho) - \nabla J(\bx_\rho,C_\rho)\right)$. Taking the norm and using the smoothness of $J$ with respect to the first variable, we get $\|\mathbf{1}^\top \nabla J(\overline{\bx},C_\rho)\| \leq \sqrt{m}\beta\|\overline{\bx} - \bx_\rho \|$. Since $\|\bx_\rho - \overline{\bx}\|  \rightarrow 0$ and $C_\rho \in U_{\overline{\bx}}$, for all $\rho \geq \rho_0$, it follows that $\|\mathbf{1}^\top \nabla J(\overline{\bx},C_\rho)\|\rightarrow 0$, contradicting the assumption. Hence, $\overline{\bx}$ is a consensus fixed point. Convergence of clusters for a finite value of $\rho$ is now a direct consequence of the fact that $\bx_\rho \rightarrow \overline{\bx}$ and Lemma \ref{lm:clust_convg}.
\end{proof}

\section{On Assumption \ref{asmpt:coerc}}\label{app:asmpt3}

In this section we provide a result which underlines the generality and wide applicability of Assumption \ref{asmpt:coerc}.

\begin{lemma}\label{lm:asmpt3}
    The $K$-means, Huber, logistic and fair loss functions all satisfy Assumption \ref{asmpt:coerc}, independent of the data.
\end{lemma}

\begin{proof}
    From the definition of each loss, it is not hard to see that all are coercive. Next, it can be readily verified that all four losses are convex, as they are compositions of a non-decreasing convex function and a convex function, i.e., we have $f(x,y) = h(g(x,y))$, where $h: \R \mapsto \R$ is convex and non-decreasing on $[0,\infty)$ and $g: \R^d \times \R^d \mapsto [0,\infty)$ is the standard Euclidean distance, i.e., $g(x,y) = \|x - y\|$.\footnote{Same results can be shown to hold for Mahalanobis distance, i.e., $g(x,y) = \|x - y\|_A$, for some positive definite matrix $A \in \R^{d \times d}$.} To verify that $\beta$-smoothness holds, we note that it suffices to show that each loss has $\beta$-Lipschitz continuous gradients, for some $\beta > 0$ (see Lemma 1.2.3 in \cite{lectures_on_cvxopt}). For $K$-means this is obvious, with $\beta = 2$. Similarly, Huber loss has $\beta$-Lipschitz continuous gradients with $\beta = 2$, see, e.g., Lemma B.2 in \cite{pmlr-v162-armacki22a}. For logistic and fair losses, we proceed as follows. From Lemma 1.2.2 in \cite{lectures_on_cvxopt}, we know that it suffices to show $\|\nabla^2_{xx} f(x,y)\| \leq \beta$, for all $x,y \in \R^d$. If $f$ is the logistic loss, i.e., $f(x,y) = \log(1+\exp(\|x-y\|^2))$, it can then be shown that $\nabla^2_{xx} f(x,y) = \frac{2I}{1+\exp(-\|x-y\|^2)} + \frac{4\exp(-\|x-y\|^2)}{[1+\exp(-\|x-y\|^2)]^2}(x-y)(x-y)^\top$, where $I \in \R^{d \times d}$ is the identity matrix, hence
    \begin{align*}
        \|\nabla^2_{xx} f(x,y)\| \leq 2 + 4\exp(-\|x-y\|^2)\|(x-y)(x-y)^\top\|,
    \end{align*} which follows from the triangle inequality. For ease of notation, let $z = x-y$ and consider $\exp(-\|z\|^2)\|zz^\top\|$. As $zz^\top \in \R^{d \times d}$ is a symmetric rank 1 matrix, with non-zero eigenvalue $\lambda = \|z\|^2$, we then $\exp(-\|z\|^2)\|zz^\top\| = \exp(-\|z\|^2)\|z\|^2$. Since the function $p(t) = t\exp(-t)$, for $t \geq 0$, reaches its global maximum at $t = 1$, it readily follows that $\|\nabla^2_{xx} f(x,y)\| \leq 2 + 4\exp(-1)$, therefore the logistic loss is $\beta$-smooth, for $\beta = 2 + 4\exp(-1)$. If $f$ is the fair loss, i.e., $f(x,y) = 2\eta^2[\|x-y\|^2/\eta - \log(1 + \|x-y\|^2/\eta)]$, for any $\eta > 0$, it can then be shown that $\nabla_{xx}^2 f(x,y) = 4\eta I - \frac{4\eta^2I}{\eta + \|x-y\|^2} + \frac{8\eta^2}{(\|x-y\|^2 + \eta)^2}(x-y)(x-y)^\top$, hence
    \begin{align*}
        \|\nabla_{xx}^2 f(x,y)\| \leq 8\eta + \frac{8\eta^2\|x-y\|^2}{(\|x-y\|^2 + \eta)^2},
    \end{align*} which follows from the triangle inequality and $\|(x-y)(x-y)^\top\| = \|x-y\|^2$. If $\|x-y\| \leq 1$, we have $8\eta^2\|x-y\|^2/(\|x-y\|^2+\eta)^2 \leq 8$. Otherwise, we have $8\eta^2\|x-y\|^2/(\|x-y\|^2+\eta)^2 \leq 8\eta^2$. Combining, we get $\|\nabla^2_{xx}f(x,y)\| \leq 8\eta + 8\max\{1,\eta^2\}$, completing the proof.
\end{proof}

\section{Additional experiments}\label{app:exp}

In this section we provide additional numerical experiments. Subsection \ref{subsec:data} provides a detailed description of the datasets and network, Subsection \ref{subsec:rho} provides experiments testing the performance for different values of penalty $\rho$, Subsection \ref{subsec:upd} provides experiments testing the performance for different values of center updates $B$, while Subsection \ref{subsec:m} tests the performance of our methods for varying number of users $m$. 

\subsection{Data and network}\label{subsec:data}

We use Iris \cite{iris}, MNIST \cite{lecun-mnist} and CIFAR10 \cite{krizhevsky2009learning} data. Iris consists of $K = 3$ classes, with $N = 150$ samples evenly split among the classes and $d = 4$ features. MNIST and CIFAR10 consist of ten classes each, with a total of $50.000$ training samples. The number of features of MNIST and CIFAR data, corresponding to pixels of images, is $d = \{784,3072\}$. We normalize the MNIST and CIFAR datasets, dividing all the pixels by the largest value, so that each pixel belongs to $[0,1]$. We use the full MNIST and CIFAR10 datasets (i.e., $K = 10$, $N = 50.000$), as well as smaller subsets. In particular, we use the first seven digits of MNIST, dubbed MNIST7 and create two CIFAR10 subsets, with three (CIFAR3) and eight (CIFAR8) classes. We consider two scenarios with respect to local data distributions: \emph{homogeneous} and \emph{heterogeneous}. In the homogeneous scenario, each user has access to all classes in equal proportion, while in the heterogeneous scenario, users have access to strict subsets of all classes, with possibly varying proportions of samples per class. In the homogeneous setup, we randomly select a thousand samples per class, for a total of $N = \{7.000, 3.000, 8.000,50.000,50.000\}$ samples for MNIST7, CIFAR3 CIFAR8 and full MNIST and CIFAR10 data. In the heterogeneous setup we use the same Iris and CIFAR3 data, while we sample a new MNIST7 dataset, to allow for varying proportion of samples per class, which consists of the first seven digits and $8.380$ samples. For heterogeneous Iris and CIFAR3 datasets, each user has access to two out of three classes, while for heterogeneous MNIST7 data, users have access to anywhere between three and five out of seven classes. Figure \ref{fig:het_dist} shows the distributions of classes and proportions of samples per class, for each user in the heterogeneous data scenario. For all our methods we use the step-size $\alpha = \frac{1}{2m|\D_{\max}|/\rho + \lambda_{\max}(L) + 1}$, where $|\D_{\max}| = \max_{i \in [m]}|\D_i|$ is the size of the largest dataset. For the centralized and local clustering we use the step-size $\alpha_c = \frac{1}{2|\D|}$ and $\alpha_l = \frac{1}{2|\D_i|}$, respectively, with $|\D|$ and $|\D_i|$ being the sizes of full data across all users and data of user $i$. Unless specified otherwise, we use a network of $m = 10$ users, communicating over a ring graph. The default communication network is visualized in Figure \ref{fig:ring-graph}. 

\begin{figure*}
\centering
\begin{adjustwidth}{-1in}{-1in} 
\begin{tabular}{lll}
\includegraphics[width=0.43\columnwidth]{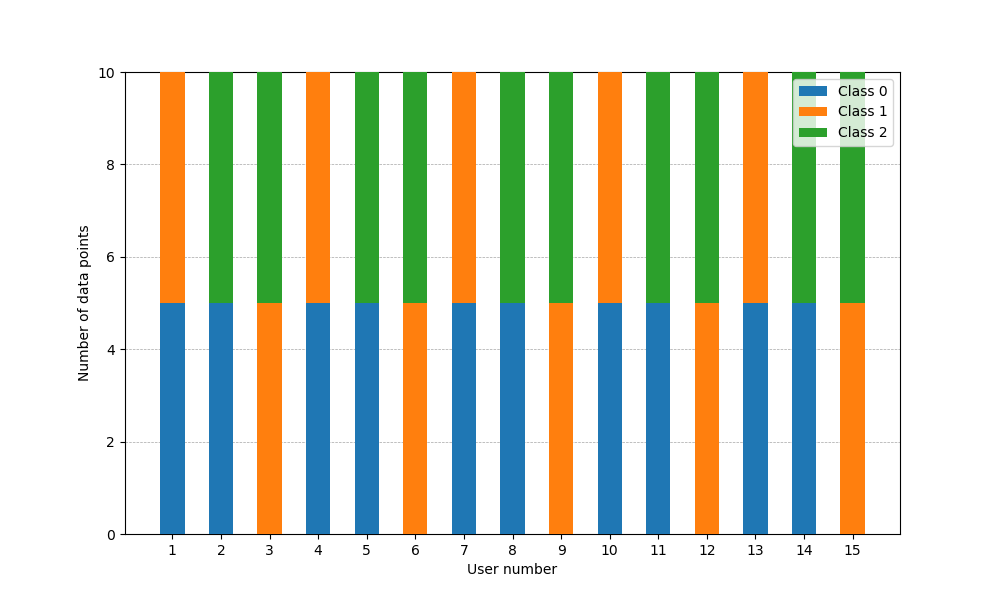}
&
\includegraphics[width=0.43\columnwidth]{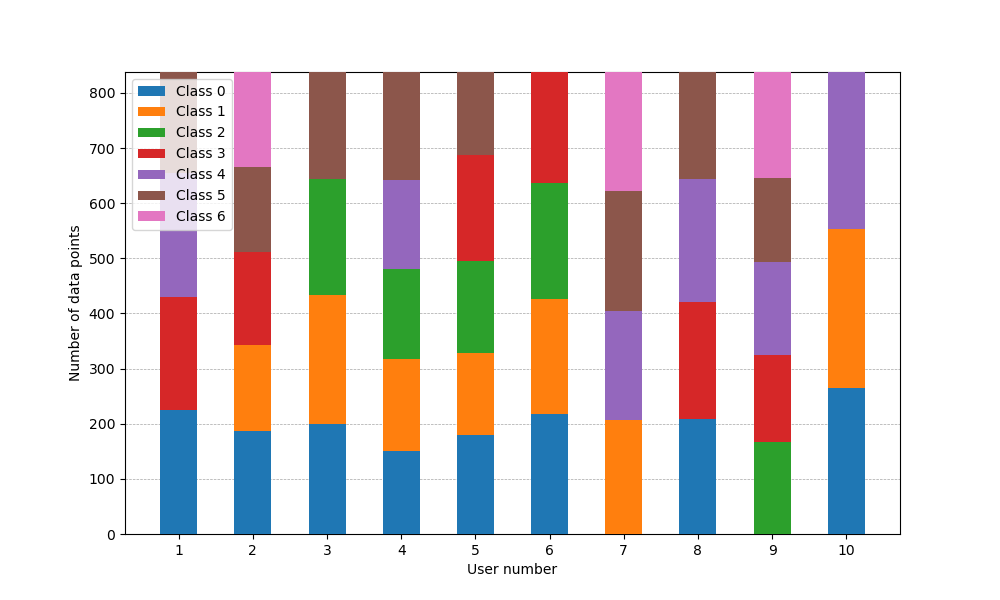}
&
\includegraphics[width=0.43\columnwidth]{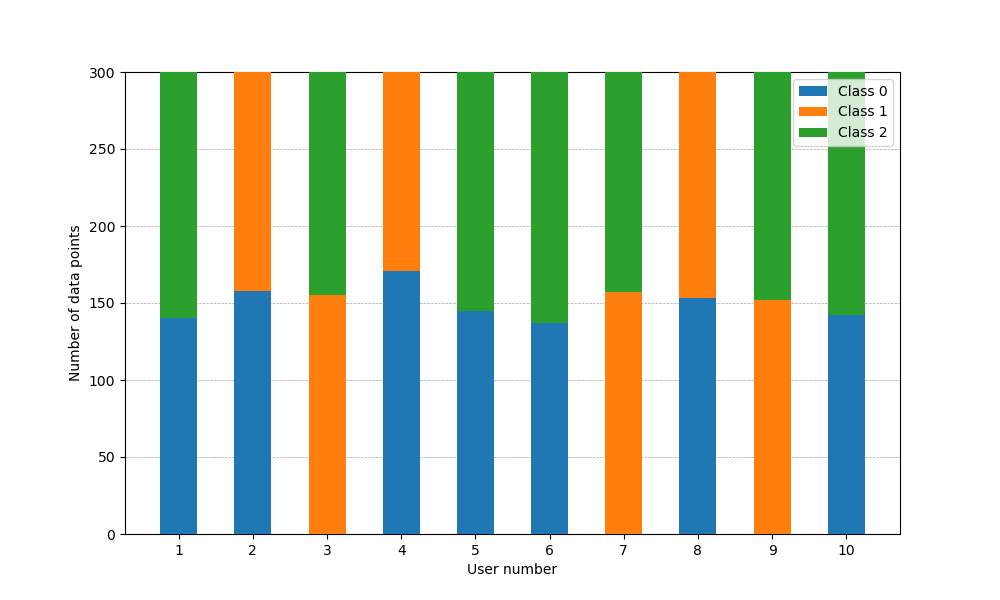}
\end{tabular}
\end{adjustwidth}
\caption{Data distribution across users in the heterogeneous data setup. The $x$ axis shows the number of users, with $y$ axis showing the number of data points per user. The bars show the classes and proportion of samples per class available at each user. Left to right: Iris, MNIST7 and CIFAR3 datasets.}
\label{fig:het_dist}
\end{figure*}

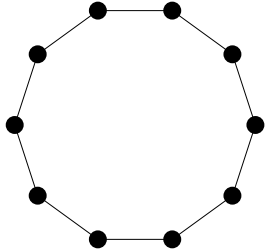
\begin{figure}[ht]
  \centering
  \scalebox{0.8}{
  \begin{tikzpicture}
    \foreach \i in {0,...,9} {
      \draw ({360/10 * \i}:2) node[circle, fill, inner sep=3pt] {} -- ({360/10 * (\i + 1)}:2);
    }
  \end{tikzpicture}
  }
  \caption{Ten users communicating over a ring graph. Unless specified otherwise, these are the default number of users and communication topology used in our experiments.}
  \label{fig:ring-graph}
\end{figure}

\subsection{Effects of $\rho$}\label{subsec:rho}

In this subsection, we test the effects of the penalty parameter $\rho$ on the performance of our algorithms. The first experiment studies the effects of $\rho$ on convergence speed. We use the homogeneous Iris dataset, with warm start initialization (i.e., centers are initialized by randomly choosing a sample from each class). We test the performance of our methods using $K$-means, Huber, Logistic and Fair losses, with $B = 1$. We run our algorithms for $T = 1.000$ iterations, computing the cost $J_\rho$ in each iteration, denoted by $J_{\rho}^t$. To normalize the data, we then subtract the final value, i.e., we plot $J_{\rho}^t - J_{\rho}^T$, as we know that the cost is decreased in each iteration. We call the quantity $J_{\rho}^t - J_{\rho}^T$ \emph{normalized cost}, and use it as the performance metric. The results are presented in Figure \ref{fig:normalized-cost}. We can see that the normalized cost converges slower for larger values of $\rho$, as discussed in Section \ref{sec:main} and as noted by \cite{kar2019clustering}. Note that initially, the methods with larger values of $\rho$ converge quickly and then slow down. This to be expected, as for larger values of $\rho$, DGC-$\mathcal{F}_\rho$ prioritizes the consensus part of cost $J_\rho$ and quickly reaches the consensus space. After that, optimizing the clustering part of the cost $J_\rho$ is slow, as the gradient becomes very small once consensus is reached, and progress slows down. On the other hand, the convergence for lower values of $\rho$ is slower in the initial phase, but becomes faster as the training progresses, as smaller values of $\rho$ provide a good balance between optimizing the consensus and clustering parts of the loss $J_\rho$, i.e., account for both parts equally.  

\begin{figure}[!htp]
\centering
\begin{tabular}{ll}
\includegraphics[width=0.4\columnwidth]{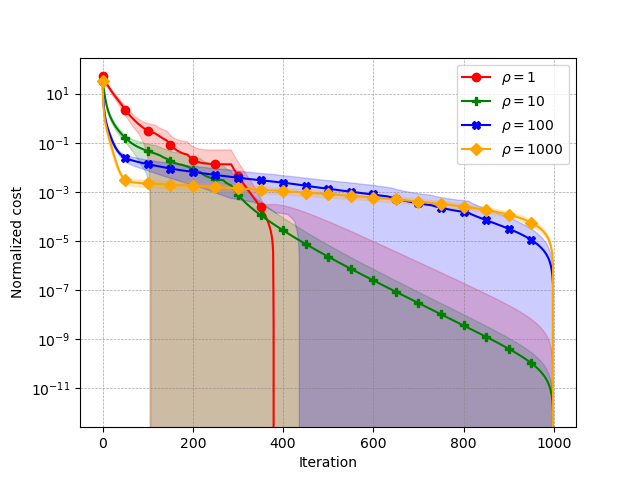}
&
\includegraphics[width=0.4\columnwidth]{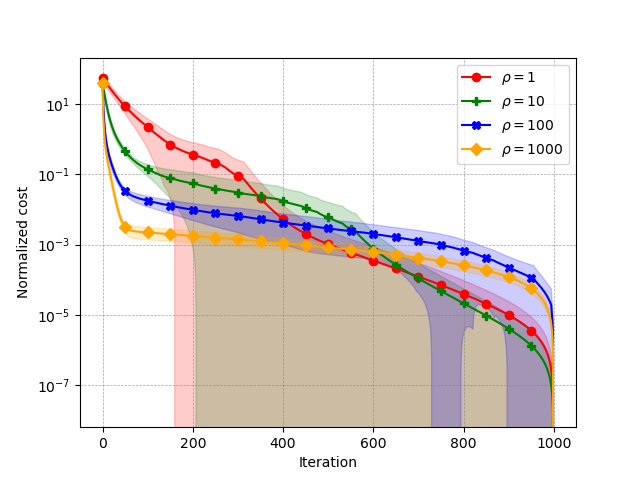}
\\
\includegraphics[width=0.4\columnwidth]{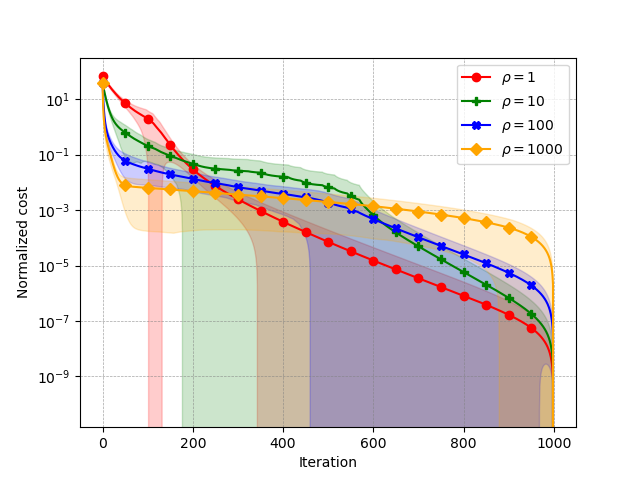}
&
\includegraphics[width=0.4\columnwidth]{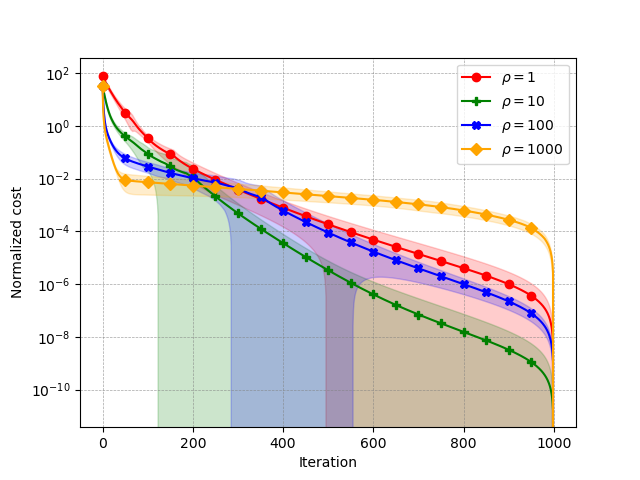}
\end{tabular}
\caption{Behaviour of $J_{\rho}^t - J_\rho^T$ for different $\rho$ and $B = 1$. Left to right and top to bottom: DGC-KM$_\rho$, DGC-HL$_\rho$, DGC-LL$_\rho$ and DGC-FL$_\rho$.}
\label{fig:normalized-cost}
\end{figure}

Next, we study the effects of $\rho$ on accuracy. The setup is the same in the previous experiment and we run the methods for $T = 500$ iterations. The results are presented in Table \ref{tab:rho_acc}. We can see that the accuracy is typically the largest for $\rho = 1$, which can be explained by the fact that for Iris data local datasets seem to be representative of the global data and for $\rho = 1$ the users strike the best balance between finding good clusters and collaborating.

\begin{table}[htp]
\caption{Effect of $\rho$ on accuracy, with $B = 1$.}
\label{tab:rho_acc}
\begin{center}
\begin{small}
\begin{sc}
\begin{tabular}{lcccc}
\toprule
 & $\rho = 1$ & $\rho = 10$ & $\rho = 100$ & $\rho = 1000$ \\
\midrule
DGC-KM$_\rho$ & $91.53 \pm 2.17\%$ & $91.13 \pm 0.85\%$ & $89.73 \pm 0.53\%$ & $91.93 \pm 0.96\%$ \\
DGC-HL$_\rho$ & $91.86 \pm 1.26\%$ & $91.20 \pm 0.78\%$ & $90.00 \pm 0.89\%$ & $91.80 \pm 1.39\%$ \\
DGC-LL$_\rho$ & $90.73 \pm 2.12\%$ & $91.00 \pm 0.45\%$ & $89.00 \pm 0.33\%$ & $90.67 \pm 1.23\%$ \\
\bottomrule
\end{tabular}
\end{sc}
\end{small}
\end{center}
\vskip -0.1in
\end{table}

\subsection{Effects of $B$}\label{subsec:upd}

In this subsection, we test the effects of the number of center updates $B$ on the performance of our algorithms. The first set of experiments, using the same setup as the ones in Section \ref{sec:num}, again aim to verify our theory, by evaluate the cost $J_\rho$ for a fixed $\rho$ and varying values of $B$. The results are presented in Figure \ref{fig:iris-e-cost}. We can see that larger values of $B$ lead to faster decrease in the cost, as predicted in Lemma \ref{lm:decr}. 

\begin{figure}[htp]
\centering
\begin{tabular}{ll}
\includegraphics[width=0.4\columnwidth]{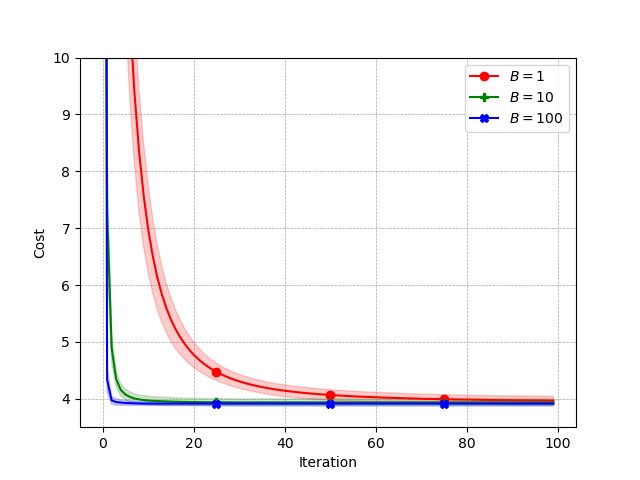}
&
\includegraphics[width=0.4\columnwidth]{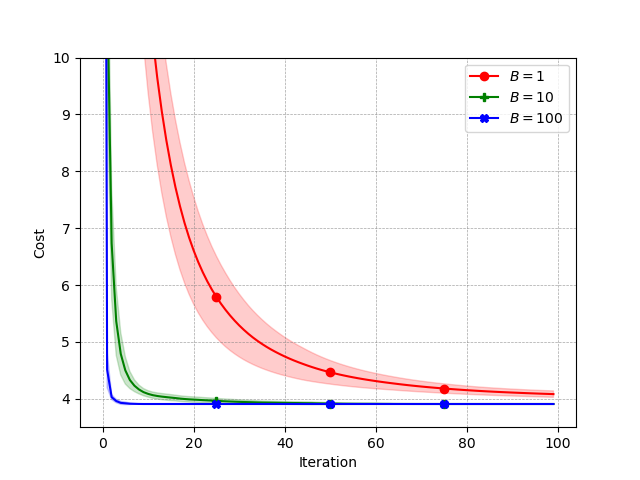}
\\
\includegraphics[width=0.4\columnwidth]{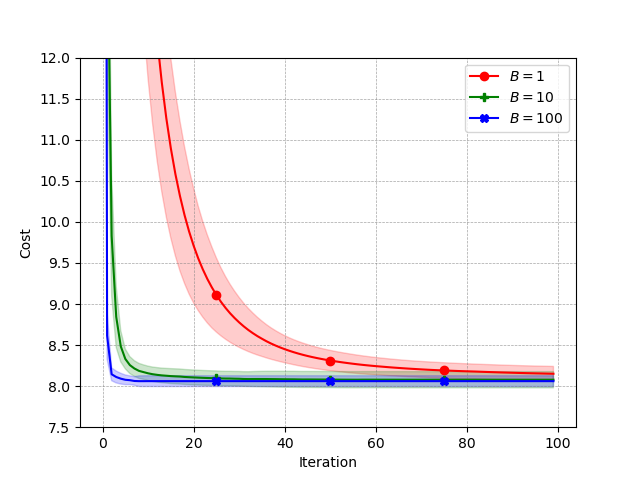}
&
\includegraphics[width=0.4\columnwidth]{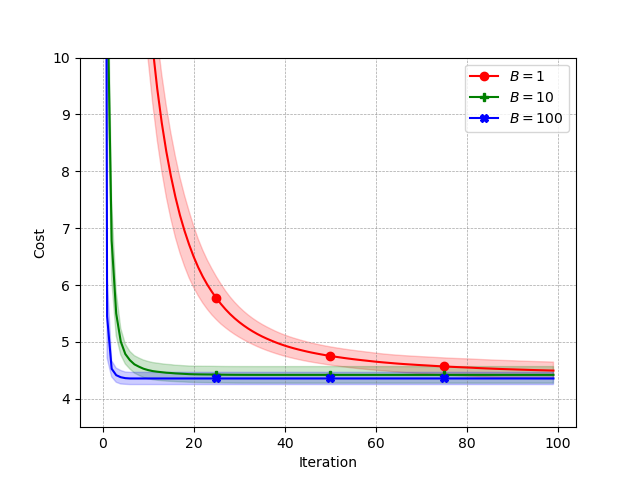}
\end{tabular}
\caption{Behaviour of $J_{\rho}$ for different $B$ and $\rho = 10$. Left to right and top to bottom: DGC-KM$_{10}$, DGC-HL$_{10}$, DGC-LL$_{10}$ and DGC-FL$_{10}$.}
\label{fig:iris-e-cost}
\end{figure}

Next, we study the effects of $B$ on accuracy, wtih $\rho = 10$. The setup is the same as in Subsection \ref{subsec:rho}. The results are presented in Table \ref{tab:e_acc}. We can see that the accuracy is not significantly affected by different values of $B$, with the number of center updates $B$ only affecting convergence speed.  

\begin{table}[htp]
\caption{Effect of $B$ on accuracy, with $\rho = 10$.}
\label{tab:e_acc}
\begin{center}
\begin{small}
\begin{sc}
\begin{tabular}{lccc}
\toprule
 & $B = 1$ & $B = 10$ & $B = 100$ \\
\midrule
DGC-KM$_{10}$ & $90.8 \pm 0.9\%$ & $90.8 \pm 0.9\%$ & $90.6 \pm 0.6\%$ \\
DGC-HL$_{10}$ & $90.6 \pm 0.5\%$ & $90.8 \pm 0.3\%$ & $90.7 \pm 0.2\%$ \\
DGC-LL$_{10}$ & $90.6 \pm 0.6\%$ & $90.7 \pm 0.5\%$ & $90.6 \pm 0.3\%$ \\
\bottomrule
\end{tabular}
\end{sc}
\end{small}
\end{center}
\vskip -0.1in
\end{table}

Finally, as discussed in the main body, e.g., Remarks \ref{rmk:comm1}-\ref{rmk:comm2}, the number of center updates $B$ offers an inherent trade-off between convergence speed and communication/computation cost. To further explore this trade-off and evaluate an optimal choice of $B$ with respect to the total communication cost incurred by our method, we perform an additional experiment on homogeneous Iris data, where we analyze the performance of our methods for different values of $B$ and present the accuracy achieved versus the number of communication rounds. We again consider a simple ring network of $m = 10$ users, with random initialization across users, in the sense that each user chooses $K = 3$ centers uniformly at random from their local data, oblivious to the true underlying cluster structure. We run all the methods for $T = 500$ iterations and perform a total of $10$ different runs. We fix $\rho = 10$ and run our methods with $B \in \{1, 5, 10, 100 \}$ center updates per iteration. Communication cost is computed as the total number of center updates $B$ performed by the method, until the given iteration, i.e., in iteration $t$, our method incurs a total communication cost $Bt$. The results are presented in Figure \ref{fig:iris-e-acc}. The left figure shows accuracy (accounting for label permutation) versus the number of iterations, while the right figure shows accuracy achieved \emph{for a fixed communication budget}. We can see from the first figure that the methods using larger number of center updates $B$ achieve a higher accuracy faster in terms of the number of iterations, with the number of iterations required to reach a certain accuracy increasing as the number of center updates $B$ decreases. This is in line with our discussion in the main body, e.g., Remarks \ref{rmk:comm1}-\ref{rmk:comm2}. However, we can see that the opposite is true from the point of communication cost, in the sense that, given a fixed communication budget, methods that perform less center updates $B$ per iteration in general achieve a higher accuracy for the allocated budget. This is again in line with our discussions and shows the duality of the parameter $B$, in that, if communication cost is not a major concern, larger $B$ should be used to maximize convergence in terms of the number of iterations, while if communication cost is a concern, smaller number of center updates per iteration should be used.
    
    \begin{figure}[htp]
    \centering
    \begin{tabular}{ll}
    \includegraphics[width=0.4\linewidth]{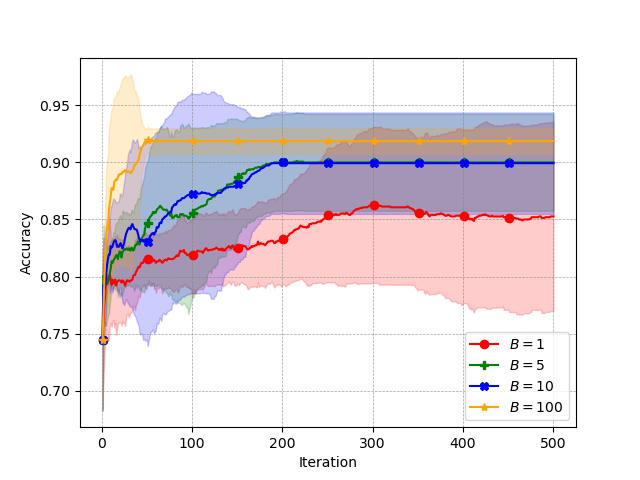}
    &
    \includegraphics[width=0.4\linewidth]{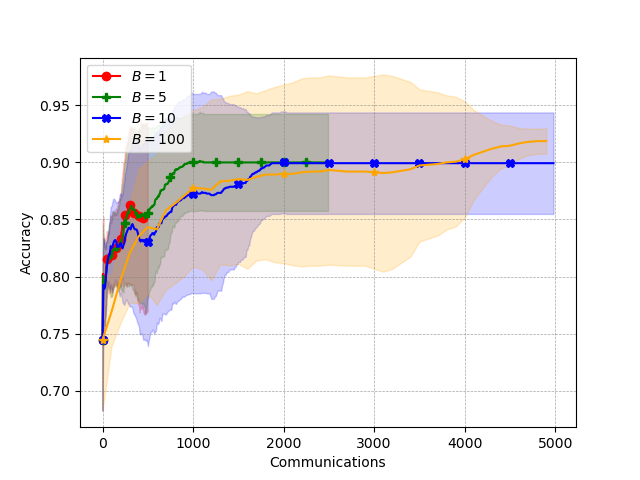}
    \\
    \includegraphics[width=0.4\linewidth]{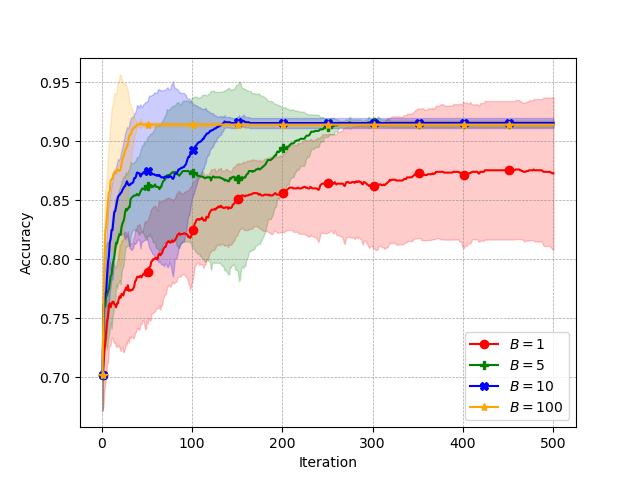}
    &
    \includegraphics[width=0.4\linewidth]{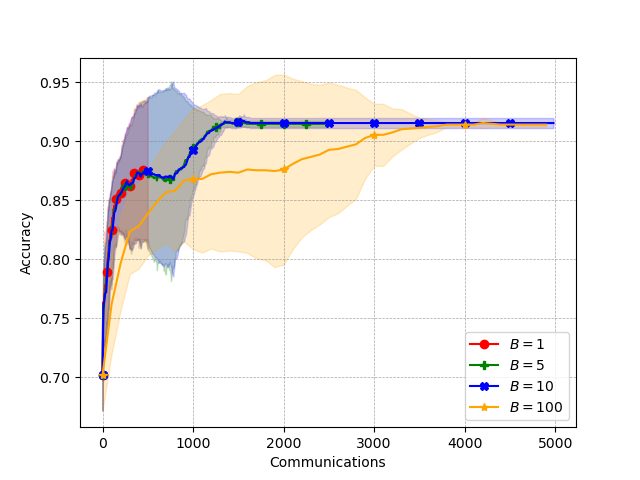}
    \\
    \includegraphics[width=0.4\linewidth]{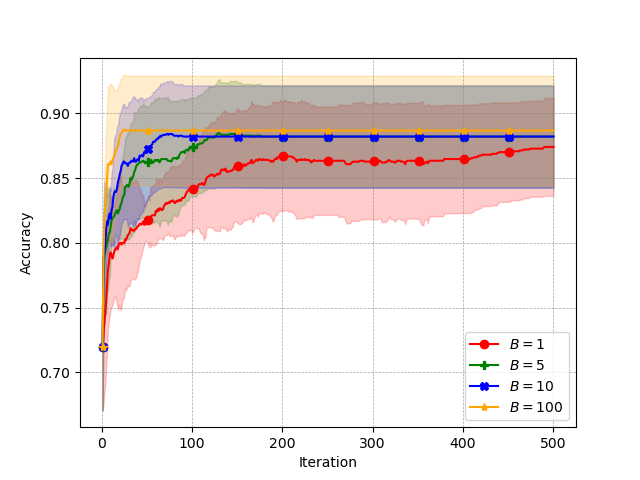}
    &
    \includegraphics[width=0.4\linewidth]{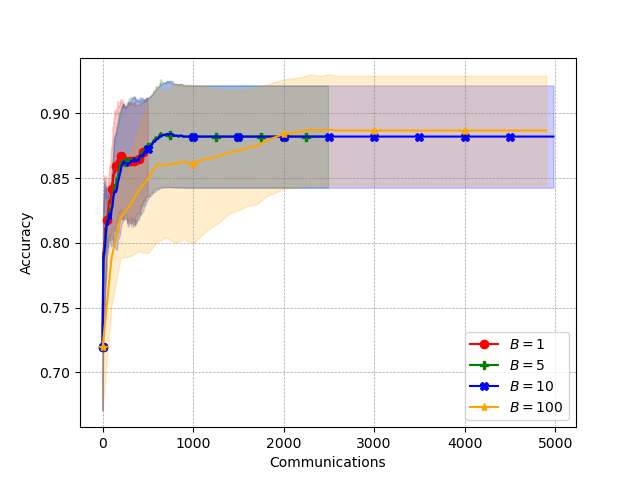}
    \\
    \includegraphics[width=0.4\linewidth]{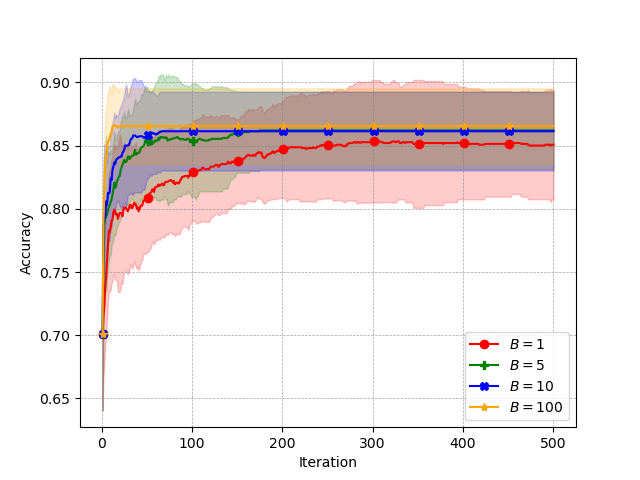}
    &
    \includegraphics[width=0.4\linewidth]{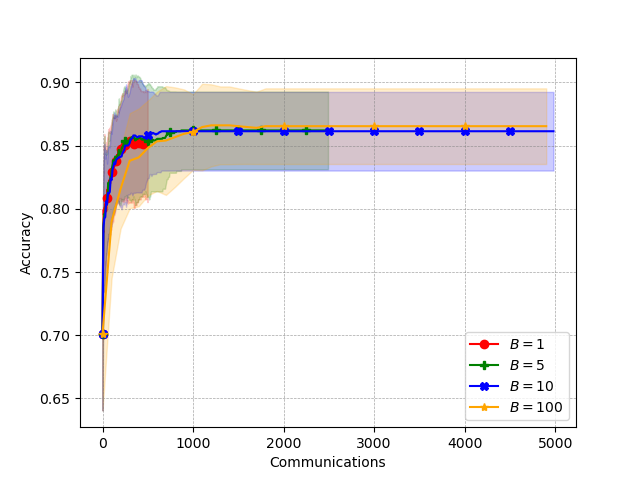}
    \\
    \end{tabular}
    \caption{Accuracy for $\rho = 10$ and different values of $B$. Left to right: accuracy versus number of iterations and accuracy versus number of communications. Top to bottom: performance of DGC-KM, DGC-HL, DGC-LL and DGC-FL.}
    \label{fig:iris-e-acc}
    \end{figure}

\subsection{Effects of $m$}\label{subsec:m}

In this subsection, we test the effects of the number of users $m$ on the performance of our algorithms. To that end, we perform an additional experiment using our DGC-KM method on the Iris dataset, where we vary the number of users from the set $m \in \{10,15,25,30\}$. For each fixed number of users $m$, we split the entire Iris dataset by randomly assigning an equal number of points to each user, without any knowledge of the underlying classes. Communication is done over a ring graph, therefore decreasing the network connectivity as the number of users grows. We set $B = 1$ and fix $\rho = 100$ and the number of iterations to $T = 4.000$, with each user initializing their centers by choosing $K = 3$ local data points uniformly at random. We average the results across five runs, for each fixed number of users. Since the data is split across users randomly, we use the ARI score to measure clustering accuracy. The results are presented in Figure \ref{fig:effect-of-m}. We can clearly see that DGC-KM achieves the same asymptotic accuracy irrespective of the number of users, with the convergence speed decreasing as the number of users increases. This is to be expected, as, recalling the discussion in Remark \ref{rmk:effect-of-m}, asymptotic accuracy depends on the initialization and the dataset itself, whereas poorer network connectivity results in slower convergence of distributed algorithms.

\begin{figure}[htp]
\centering
\includegraphics[width=0.6\textwidth]{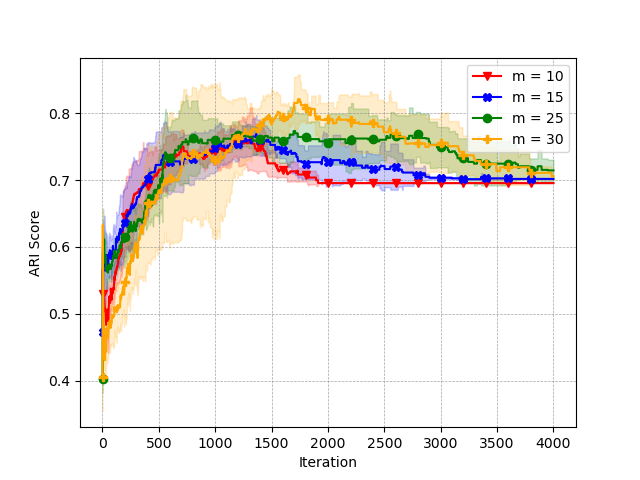}
\caption{ARI score of DGC-KM$_{100}$ for varying number of users.}
\label{fig:effect-of-m}
\end{figure}

\end{document}